%% file: deva.tex
\documentclass{article}

\usepackage{microtype}
\usepackage{graphicx}
\usepackage{subcaption}
\usepackage{booktabs} 
\usepackage{tabularx} 
\usepackage{hyperref}
\usepackage{array}
\newcolumntype{C}[1]{>{\centering\arraybackslash}p{#1}}

\usepackage[accepted]{icml2026}

\usepackage{amsmath}
\usepackage{amssymb}
\usepackage{mathtools}
\usepackage{amsthm}
\usepackage{thmtools}
\usepackage{thm-restate}
\usepackage{etoc}

\usepackage[capitalize,noabbrev]{cleveref}

\theoremstyle{plain}
\newtheorem{theorem}{Theorem}[section]

\newtheorem{corollary}[theorem]{Corollary}
\theoremstyle{definition}
\newtheorem{definition}[theorem]{Definition}
\newtheorem{example}[theorem]{Example}
\newtheorem{assumption}[theorem]{Assumption}
\theoremstyle{remark}
\newtheorem{remark}[theorem]{Remark}

\usepackage[textsize=tiny]{todonotes}

\icmltitlerunning{Decoupled Variance Adaptation}

\input{commands}

\makeatletter
\def\addcontentsline#1#2#3{%
  \addtocontents{#1}{\protect\contentsline{#2}{#3}{\thepage}{section*.\thepage}\@empty}%
}
\makeatother

\addtocontents{toc}{\protect\setcounter{tocdepth}{-1}}

\begin{document}

\twocolumn[
  \icmltitle{Decoupling Variance and Scale-Invariant Updates in Adaptive Gradient Descent for Unified Vector and Matrix Optimization}



  \icmlsetsymbol{equal}{*}

  \begin{icmlauthorlist}
    \icmlauthor{Zitao Song}{yyy}
    \icmlauthor{Cedar Site Bai}{yyy}
    \icmlauthor{Zhe Zhang}{comp}
    \icmlauthor{Brian Bullins}{yyy}
    \icmlauthor{David F. Gleich}{yyy}
  \end{icmlauthorlist}

  \icmlaffiliation{yyy}{Department of Computer Science, Purdue University, West Lafayette, USA}
  \icmlaffiliation{comp}{Edwardson School of Industrial Engineering, Purdue University, West Lafayette, USA}

  \icmlcorrespondingauthor{Zitao Song}{song903@purdue.edu}
  \icmlcorrespondingauthor{David F. Gleich}{dgleich@purdue.edu}

  \icmlkeywords{Machine Learning, ICML}

  \vskip 0.3in
]



\printAffiliationsAndNotice{}  

\begin{abstract}

  Adaptive methods like Adam have become the \textit{de facto} standard for large-scale vector and Euclidean optimization due to their coordinate-wise adaptation with a second-order nature. More recently, matrix-based spectral optimizers like Muon \citep{jordan2024muon} show the power of treating weight matrices as matrices rather than long vectors. Linking these is hard because many natural generalizations are not feasible to implement, and we also cannot simply move the Adam adaptation to the matrix spectrum. To address this, we reformulate the AdaGrad update and decompose it into a variance adaptation term and a scale-invariant term. This decoupling produces \textbf{DeVA} (\textbf{De}coupled \textbf{V}ariance \textbf{A}daptation), a framework that bridges between vector-based variance adaptation and matrix spectral optimization, enabling a seamless transition from Adam to adaptive spectral descent. Extensive experiments across language modeling and image classification demonstrate that DeVA consistently outperforms state-of-the-art methods such as Muon and SOAP \citep{vyas2024soap}, reducing token usage by around 6.6\%. Theoretically, we show that the variance adaptation term effectively improves the blockwise smoothness, facilitating faster convergence. Our implementation is available at \url{https://github.com/Tsedao/Decoupled-Variance-Adaptation}

\end{abstract}

\section{Introduction}

%

A fundamental acceleration strategy in modern optimization involves adapting stochastic gradient descent to second-order curvature \citep{amari1998natural}. This adaptation enables accelerated progress along specific coordinates while suppressing exploding updates in others. 
Adaptive variants such as AdaGrad \citep{duchi2011adaptive}, RMSprop \citep{tieleman2012lecture}, and Adam \citep{adam2014method} can be viewed as Euclidean gradient descent equipped with inexpensive approximations to  curvature information. Among these, Adam has become the \textit{de facto} standard for large-scale training, particularly for large language models (LLMs) \citep{dubey2024llama, liu2024deepseek}.

\begin{figure}[t]
    \centering
    \includegraphics[width=0.9\linewidth]{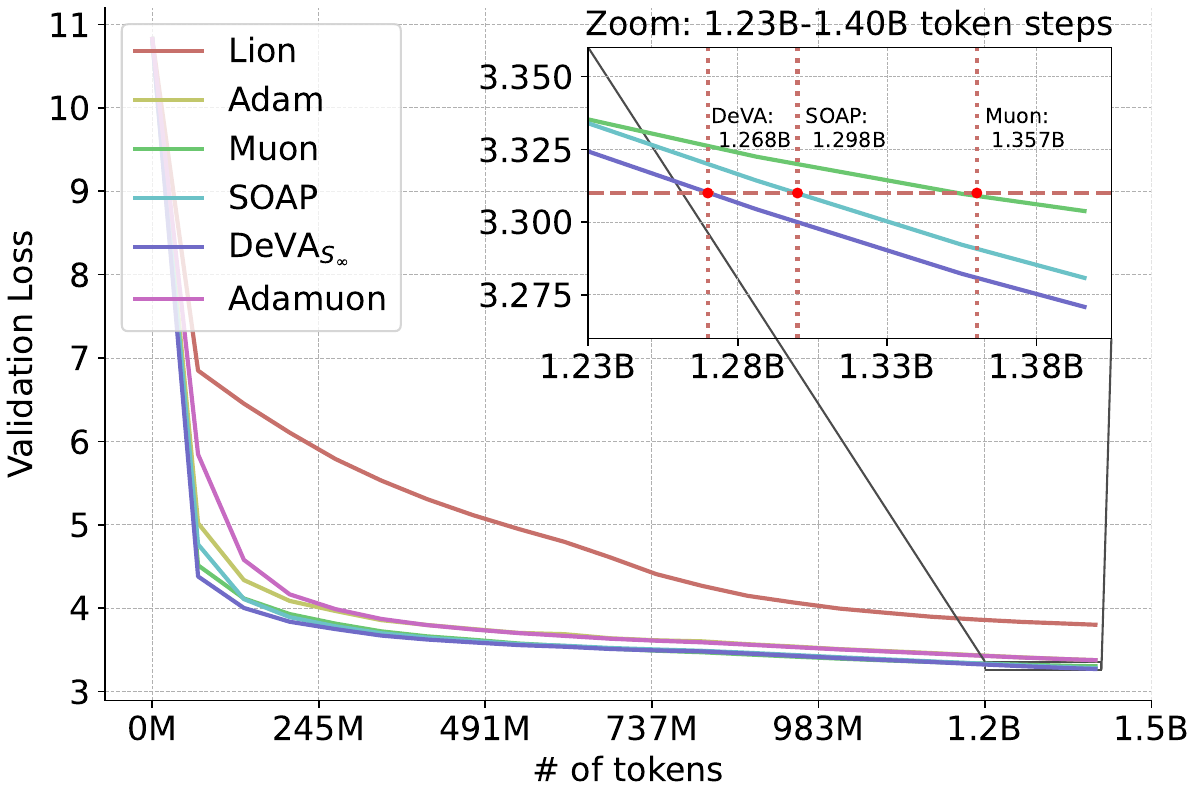}
    \caption{NanoGPT pretraining on FineWeb at a uniform $0.001$ learning rate. Compared to non-adaptive Muon, the adaptive method SOAP and \textbf{our method DeVA$_{S_{\infty}}$} achieve the target validation perplexity using $4.3\%$ and $6.6\%$ fewer tokens, respectively. \vspace{-20pt}}
    \label{fig:nano_val}
\end{figure}

\begin{table*}[!t]
  \centering
  \caption{Special instantiations of Adaptive Subgradient descent $M^{-1}g$ and $M^{-1}G$ using preconditioner $M^{-1}$ where $g\in\R^d$ and $G \in \R^{n \times m}$. Eigenbases are the rotation matrices we applied to the gradient $G$ and Decoupled Variance is the adaptive stepsize we decoupled from sign functions in \cref{eq:deva}. Here, $\sigma_i$ is the $i$-th singular value of $G$, and $\odot$ is the Hadamard / elementwise product. }
  \label{tab:summary}
  \begin{tabular*}{\textwidth}{@{\extracolsep{\fill}} p{0.23\textwidth}  C{0.36\textwidth}  C{0.09\textwidth}  C{0.22\textwidth} }
  \toprule
 \textbf{Method} & \textbf{Preconditioner $M^{-1}$} & \textbf{Eigenbases} & \textbf{Variance Adaptation} \\
  \midrule 
  {\mbox{\textbf{Adam}~\citep{adam2014method}}}   & $\E[\diag(gg^{T})]^{-1/2}$       & $I$    & $(\ema[g^2] / g^2)^{-1/2}$ \\ 
  {\mbox{\textbf{Shampoo}~\citep{gupta2018shampoo}}} & $(\E[GG^{T}]^{1/2} \otimes \E[G^{T}G]^{1/2})^{-1/2}$ & $I$ & None \\
  {\textbf{SOAP} \citep{vyas2024soap}} & $((\E[GG^{T}] \otimes \E[G^{T}G]) / \Tr(\E[GG^{T}]))^{-1/2}$ & $Q_L, Q_R$ & $(\ema[G\!\odot\! G]/G\!\odot\! G)^{-1/2}$ \\
  {\textbf{SPlus} \citep{frans2025stable}} & $(\E[GG^{T}]^{1/2} \otimes \E[G^{T}G]^{1/2})^{-1/2}$ & $Q_L, Q_R$ & None \\
  \textbf{DeVA$_{S\infty}$ (Ours)} & $(\E[(GG^{T})^{1/2} \otimes (G^{T}G)]^{1/2})^{-1/2}$ & $Q_L, Q_R$ & $(\ema[\sigma_i\sigma_j])/\sigma_i\sigma_j)^{-1/2}$ \\
  \bottomrule
  \end{tabular*}
\end{table*}

A simplifying aspect of adaptive gradient methods is that they ignore \emph{matrix structure} within the parameters and treat them as a single long vector. Muon \citep{jordan2024muon}, a spectral gradient method for layerwise matrices, showed that adaptive gradient methods could be improved upon by using this matrix structure and treating matrix parameters as matrices within the optimization. Muon has since demonstrated superior LLM pretraining performance over Adam \citep{team2025kimi, liu2025muon, wen2025fantastic}. Although some may argue, we regard Muon as a non-adaptive method as we discuss below. 

\emph{This presents an intriguing challenge: how do we combine adaptive updates with matrix-based variables? Doing so efficiently should (and does!, see~\cref{fig:nano_val}) accelerate spectral descent, mirroring its impact on Euclidean gradient descent.}

Efforts in this direction include AdaGO \citep{zhang2025adagrad}, Adamuon \citep{si2025adamuon}, and SPlus-then-Adam \citep{frans2025really}, which attempt to incorporate Adam’s second-order moments directly into orthogonal updates in Muon. While this idea is pragmatic, a satisfactory theoretical framework for applying Adam’s adaptive stepsizes to matrices remains elusive. Consequently, it is essential for us to understand the origin of adaptive step sizes in vector-based updates and extend them to matrices smoothly.


In the context of vector updates, an interpretation from \citet{balles2018dissecting,orvieto2025search} characterizes Adam as an adaptive and smoothed variant of Signum \citep{bernstein2018signsgd}. In this framework, the adaptive step size is modulated by the local \textit{signal-to-noise} ratio (SNR). Stating this requires the standard $m_t$, which is the exponential moving average of the gradient that is used in Adam, as well as the term $\sigma_t^2$, which is the exponential moving average of the variance $(m_{t-1} - g_t)^2$. These two terms allow the Adam update to be written as\footnote{Here, $\eta$ is a standard learning rate step size that applies to all coordinates. In future sections we drop this piece of the update.} 
\begin{align}
\label{eq:adam_connection}
\Delta x_t = - \underbrace{\frac{\eta}{\sqrt{1 + \sigma_t^2 / m_t^2}}}_{\text{adaptive stepsize}} \underbrace{\sign(m_t)\vphantom{\frac{\eta}{\sqrt{1 + \sigma_t^2 / m_t^2}}}}_{\text{direction}}.
\end{align}
This formulation allows a decoupling of the direction and the adaptive stepsize. Crucially, the gradient variance term $\sigma_t^2$ only arises in the stepsize.  

The Muon update can be written as a non-adaptive update with the  \emph{matrix sign} function \citep{chen2025muon}, which we call $\msign$. Analytically, the particular form used involves applying the \emph{sign} function to the vector of singular values in an SVD and then recombining. The result is an equally weighted set of all the orthogonal directions in the data (see~\cref{eq:msign}). In Muon, 
\begin{align}
     \Delta X_t = -\eta \msign(M_t),
\end{align}
where $M_t$ is a first-order and extra Nesterov momentum combination of gradient matrices. According to \citet{jordan2024muon}, the choice of $\msign$ here is hypothesized to help by amplifying ``rare directions", those with small singular values that are nonetheless critical for learning. More recently, \citet{davis2025spectral} shows that $\msign$ will be helpful when the matrix gradient $G_t$ has a large Nuclear-to-Frobenius norm ratio.




In the context of these two expressions for Adam and Muon, we wish to explore whether one can derive an adaptation matrix $\Gamma_t$ that \emph{functions as} the adaptive stepsize in Adam but works for a matrix spectrum. Formally, we want to ask can we derive an update 
\begin{align}
\Delta X_t = -\underbrace{\eta \Gamma_t}_{\text{matrix adaptive stepsize}} \odot \underbrace{\msign(G_t)}_{\text{matrix directions}},
\end{align}
where $\odot$ is the Hadamard product?

In this work, we systematically address this idea by introducing \textbf{DeVA} (\textbf{De}coupled \textbf{V}ariance \textbf{A}daptation), a unified framework that interprets adaptive gradient methods as preconditioned gradient descent. As a result, popular algorithms could be viewed as special instances of the general framework with different instantiations of the preconditioners $M$ (see \cref{tab:summary}).  Based on that, \textbf{DeVA} derives coordinate-wise adaptive stepsizes naturally for matrix parameters by decoupling a gradient variance adaptation and a scale-invariant update from the underlying preconditioned update. 
Our contributions are threefold:

\textit{Unification:} Under a diagonal curvature approximation 
we derive $\text{DeVA}_{\ell_\infty}$ and establish its fundamental connection to Adam.  Extending this to structured matrix updates, we utilize a Kronecker approximation 
to introduce $\devasf$, a novel adaptive spectral gradient descent method that enables curvature-aware updates throughout the training process (see \cref{thm:deva_coordinate_wise}). This framework further generalizes to other preconditioned methods, such as Shampoo \citep{gupta2018shampoo}, SOAP \citep{vyas2024soap}, and SPlus \citep{frans2025stable} when the expectation of the Kronecker product is separable.

\textit{Numerical Validation:} Our experiments, ranging from synthetic optimization benchmarks to language modeling (see \cref{fig:nano_val}) and large-scale image classification (see \cref{fig:imagenet_loss}) demonstrate that $\text{DeVA}_{S_\infty}$ consistently outperforms its non-adaptive counterpart Muon and adaptive spectral descent competitors SOAP and Adamuon. 


\textit{Theoretical Implication:} Finally, under the blockwise smoothness assumption, our theoretical analysis justifies the necessity of incorporating adaptive stepsize into matrix-based optimizers by showing it helps to reduce the blockwise smoothness that appears inside the convergence result (see \cref{thm:devalf} and \cref{thm:devasf}).
\paragraph{Organization.} \cref{sec:preliminary} provides preliminaries and key related work on descent methods. \cref{sec:methods,sec:analysis} introduce and analyze the DeVA framework and its variants ($\devalf$, $\devasf$). Numerical evaluations follow in \cref{sec:experiment}, while additional related work is delegated to Appendix \ref{appsec:related_work}.

\section{Preliminaries}
\label{sec:preliminary}

\paragraph{Notation.}
We use lowercase italics ($x, g$) for vectors/scalars and uppercase letters ($X, G$) for matrices. Powers are \emph{matrix-powers}, for example, $X^{-1/2}$ denotes the matrix inverse square root unless explicitly stated otherwise. Subscripts $x_t$ denote iterates, and $\ema[\cdot]$ signifies the exponential moving average. We denote the inner product as $\inner{\cdot}{\cdot}$, while $\odot, \oslash$, and $\otimes$ represent the Hadamard product, Hadamard division, and Kronecker product, respectively. We use $e_i$ as the $i$-th standard basis vector, and $1_m$ is the $m$-dimensional all-ones vector. Where clear from context, we'll use $g_t$ to be the function gradient at the $t$th iteration for the vector case and $G_t$ to be the function gradient for a matrix variable. 

\paragraph{Steepest Gradient Descent.} Given a general function $f : \R^d \to \R$ and a Mahalanobis norm defined on a positive-definite matrix $M$, i.e., $\norm{\cdot}_{M} = \sqrt{\inner{\cdot }{M\cdot}}$, let $g=\nabla f(x)$. Then steepest gradient descent under this norm generates a direction $\Delta x^{*}$ by minimizing a local quadratic approximation of $f$ at $x$:
\begin{align}
    \Delta x ^{*} = \argmin_{\Delta x} \inner{g}{\Delta x} + \frac{1}{2} \norm{\Delta x}_{M}^2.
\end{align}
Solving this yields $\Delta x^{*}= -M^{-1} g$, where $M^{-1}$ is often referred to as a \textit{preconditioner}. This formulation is the basis for adaptive subgradient methods \citep{duchi2011adaptive}. 

In AdaGrad \citep{duchi2011adaptive}, $M$ is defined as $(\sum_{t} g_t^{} g_t^{T})^{1/2}$, a choice motivated by regret minimization. We will see many $gg^T$ operations and it simplifies notation to let $H = gg^T$. More recently, $M$ has been interpreted as a \textit{whitening metric} \citep{frans2025stable,eschenhagen2026clarifying}, defined as the square root of the gradient covariance, $M_{\text{Whitening}} := \E[H]^{1/2}$ (Unless stated otherwise, $\E[\cdot]$ denotes expectation over minibatches), so that $M_{\text{Whitening}}^{-1}$ normalizes the covariance of the gradient vector $g$. This yields the whitening update 
\begin{align}
\label{eq:sgd}
    \Delta x^{*}_{\text{whitening}} = -\E[H]^{-1/2} g.
\end{align}
The setup leads to some nice connections such as an equivalence with the Fisher Information Matrix $F(x)$ when $f$ is a negative log-likelihood function, in which case the Hessian $\gH(x):=\nabla^2 f(x)$ also coincides with $F(x)$. 
Under this scenario, the $\E[H]$ term can further be considered as the Gauss-Newton or the covariance matrix of the gradient \citep{martens2020new,morwani2024new}. Notably, Adam \citep{adam2014method} (or  RMSprop \citep{tieleman2012lecture} with zero first-order momentum) can be viewed as a computationally efficient realization of \cref{eq:sgd}, where $H$ is restricted to a diagonal approximation $\diag(g g^T)$ and the expectation is replaced by an exponential moving average (EMA). This avoids the intractability of computing matrix square roots in high-dimensional parameter spaces.

\paragraph{Matrix-based Descent.} For a function $f : \R^{n \times m} \to \R$ with parameter $X \in \R^{n \times m}$ and gradient $G \in \R^{n \times m}$, a standard AdaGrad-style update requires maintaining a prohibitively large $H \in \R^{nm \times nm}$ by vectorizing $G$. To circumvent this, Shampoo \citep{gupta2018shampoo} approximates the covariance matrix $\E[\fvec(G)\fvec(G)^{T}]$ by accumulating two smaller factorized covariance matrices, $\E[L] \in \R^{n \times n}$ and $\E[R] \in \R^{m \times m}$, where $L=GG^{T}$ and $R=G^{T}G$. The resulting Shampoo update $\Delta X^{*}$, expressed in the form of \cref{eq:sgd}, is given by:
\begin{align}
\label{eq:shampoo}
    \fvec(\Delta X^{*}_{\text{Shampoo}}) = -\!\left(\E [L]^{1/2} \otimes \E[R]^{1/2}\right)^{-1/2} \fvec(G).
\end{align}
Similarly, SOAP \citep{vyas2024soap, morwani2024new} approximates the Gauss-Newton matrix by computing an optimal Kronecker product approximation, which yields:
\begin{align}
  \fvec(\Delta X^{*}_{\text{SOAP}}) = -\! ((\E [L] \otimes \E[R])/\Tr(\E[L]))^{-1/2} \fvec(G).
\end{align}
If we consider the gradient sample by removing the expectation in \cref{eq:shampoo}, the term $(L^{1/2} \otimes R^{1/2})^{-1/2} \text{vec}(G)$ can be simplified. Using the identity $(A \otimes B)^k = A^k \otimes B^k$ and the vectorization property $(A \otimes B) \fvec(X) = \fvec(AXB^{T})$, the update simplifies to $\Delta X^* = -L^{-1/4} G R^{-1/4}$. Substituting the singular value decomposition $G = U\Sigma V^T$, the update becomes:
\begin{align}
\label{eq:msign}
    -\Delta X^{*}_{\text{SpectralGD}} = (GG^{T})^{-1/4}G(G^{T}G)^{-1/4} = UV^{T}.
\end{align}
This result recovers spectral gradient descent methods \citep{carlson2015stochastic, jordan2024muon}, which effectively orthogonalize the gradient via the matrix sign function. Consequently, by setting $H = L^{1/2} \otimes R^{1/2}$, update $-H^{-1/2} \text{vec}(G)$ indeed is performing an orthogonalization step. If we plug this $H$ back to the whitening update in \cref{eq:sgd},  we can reframe it as a Shampoo-style update:
\begin{align}
\label{eq:shampoo_reframed}
\fvec(\Delta X^{*}_{\text{Shampoo-style}}) = -\E[L^{1/2} \otimes R^{1/2}]^{-1/2} \fvec(G).
\end{align}

\section{Our Methods}
\label{sec:methods}
In this section, we present the \textbf{De}coupled \textbf{V}ariance \textbf{A}daptation (\textbf{DeVA}) framework. To motivate our approach, consider the whitening update in \cref{eq:sgd}: $\Delta x^* = -\mathbb{E}[H]^{-1/2}g$. As established in previous sections, if $H=\diag(gg^{T})$, $\E[H]$ becomes, under the EMA approximation, the adaptive step size in Adam and RMSprop. In the matrix case where $H=L^{1/2} \otimes R^{1/2}$, $\E[H]$ then becomes the Shampoo-style update.

If we remove the expectation of $H$, we make a key observation. Specifically, $H^{-1/2}g$ yields a sign-based update: it recovers $\sign(g)$ when $H=\diag(gg^{T})$ and the matrix sign $\msign(G)$ when $H=L^{1/2} \otimes R^{1/2}$. This suggests that the scale-invariant update, $\sign(g)$ or $\msign(g)$, is indeed embedded within the preconditioned update $H^{-1/2}g$.

Motivated by that, we define the DeVA framework by refactoring the whitening update in \cref{eq:sgd} as follows:
\begin{align}
\label{eq:deva}
\Delta x^{*}_{\text{whitening}} = -\underbrace{\left( \mathbb{E}[H]^{-1/2} H^{1/2} \right)}_{\text{Decoupled Variance Adaptation}} \cdot \underbrace{\left( H^{-1/2} g \right)}_{\text{Scale-Invariant update}},
\end{align}
This decomposition provides an appealing representation of adaptive step sizes for gradient descent methods with a scale-invariant update. Here we call the term $\mathbb{E}[H]^{-1/2} H^{1/2}$ as \textit{variance adaptation}. When $H=gg^{T}$, this \textit{variance adaptation} is also called the \textit{signal-to-noise ratio} (SNR) of $g$ in \citet{balles2018dissecting,orvieto2025search}, i.e., $(gg^{T}/\E[gg^{T}])^{1/2}$, where a larger SNR implies a larger update in the steepest decent direction.  In the following subsections, we first establish that the DeVA framework recovers Adam in the vector setting; we then generalize this framework to matrix optimization, introducing a novel adaptive spectral descent method.

\begin{algorithm}[!t]
\caption{$\deva_{\ell_{\infty}}$}\label{alg:deva_linfty}
\begin{algorithmic}[1]
\INPUT $\beta_1$, $\beta_2$, $\epsilon$, $T$, $\eta_t$ for $t=1,\dots,T$
\OUTPUT $x_T \in \R^d$
\STATE Initialize $m_0, v_0$ to $0$
\STATE Randomly initialize $x_1 \in \R^d$
\FOR{$t=1,...,T$}
\STATE $g_{t} = \nabla f(x_{t},\xi_{t}) $
\STATE $m_{t}= \beta_1 m_{t-1} + (1-\beta_1)g_{t} $
\STATE $v_t = \beta_2 v_{t-1} + (1-\beta_2) {\color{blue} m_t^2}$ \emph{\textcolor{blue}{$\leftarrow$ differs from Adam}}  \label{line:devalf}
\STATE $\gamma_t = \left(\frac{v_t}{ m_t^2 + \epsilon }\right)^{-1/2}$
\STATE $d_t =\gamma_t \odot \sign(m_t)$
\STATE $x_{t+1} = x_t - \eta_t d_t$
\ENDFOR
\end{algorithmic}
\end{algorithm}

\subsection{Connection to Adam}
For vectors $\Delta x, g \in \R^d$, let $H = \diag(gg^T)$. We can express the DeVA framework from \cref{eq:deva} in an element-wise form as:
\begin{align}
\label{eq:deva_linfty}
    \Delta x^{*}_{\devalf} = -\underbrace{\bigg(\frac{\E [g^2]}{g^2} \bigg)^{-1/2}}_{\text{Decoupled Variance Adaptation}} \odot \sign(g),
\end{align}

where $g^2$ denotes the element-wise square and division is element-wise. In this formulation, we have successfully decoupled the gradient variance term from the standard sign-based update. Since $\text{sign}(g)$ is the solution to the constrained linear minimization problem under an $\ell_{\infty}$-norm ball, i.e., $\sign (g):=-\argmin_{\norm{\Delta x}_{\ell_{\infty}} \le 1} \inner{g}{\Delta x} $, we refer to this specific instance as $\devalf$.

Notably, $\devalf$ recovers the Adam optimizer when the gradient expectation is approximated by the EMA under the bias correction \citep{adam2014method}. Recall that the Adam update (or RMSprop with zero first-order momentum) is given by:
\begin{align}
    \Delta x^{*}_{\text{Adam}} =- \frac{g}{\sqrt{\ema[g^2]}} =-\bigg(\frac{\ema[g^2]}{g^2}\bigg)^{-1/2} \odot \sign(g).
\end{align}
When first-order momentum is incorporated, Adam replaces the raw gradient $g$ with its moving average $m_t = \ema[g]$. Similarly, in the DeVA framework, we replace $g$ with its expectation $\E[g]$ in \cref{eq:deva_linfty}. In practice, since evaluating the true expectation $\E[g]$ is computationally prohibitive, $\devalf$ also utilizes the $\ema[g]$ to approximate $\E[g]$.

\begin{algorithm}[!t]
\caption{$\devasf$}\label{alg:deva_sinfty}
\begin{algorithmic}[1]
\INPUT $\beta_1, \beta_2, \beta_3, \text{freq}$, $\epsilon, T$, $\eta_t$ for $t=1,...,T$
\OUTPUT $X_T \in \R^{n \times m}$
\STATE Initialize $L_0, R_0, M_0, V_0 $ to $0$
\STATE Randomly initialize $X_1 \in \R^{n \times m}$
\FOR{$t=1,...,T$}
\STATE $G_{t} = \nabla f(X_{t},\xi_{t}) $
\STATE $L_t = \beta_3 L_{t-1} + (1-\beta_3) G_tG_t^{T}$
\STATE $R_t=\beta_3 R_{t-1} + (1-\beta_3)G_t^{T}G_t$
\IF{$t \text{ mod } \text{freq is }  0$} 
\STATE $Q_L,Q_R = \texttt{eigdecomp}(L_t), \texttt{eigdecomp}(R_t)$ \label{line:deva_eigde}
\ENDIF
\STATE $G_t'=Q_L^{T}G_tQ_R$
\STATE $M_{t}= \beta_1 M_{t-1} + (1-\beta_1)G_{t}' $
\STATE $r_t = \left( \sqrt{\sum_{j=1}^m (M_t)_{ij}^2} \right)_{i=1}^n$ 
\STATE $c_t = \left( \sqrt{\sum_{i=1}^n (M_t)_{ij}^2} \right)_{j=1}^m$
\STATE $V_t = \beta_2 V_{t-1} + (1-\beta_2) {\color{blue} r_t^{} c_t^{T}}$ \emph{\textcolor{blue}{$\leftarrow$ spectral adaptivity}} \label{line:deva_second_update}
\STATE $\Gamma_t = \left(V_t \oslash (r_t^{} c_t^{T} +\epsilon)\right)^{-1/2}$
\STATE $D_t =\Gamma_t \odot \msign(M_t) \odot  0.2 \sqrt{\max (n,m)}$\label{line:rms_normalize}
\STATE $X_{t+1} = X_t - \eta_t Q_L D_t Q_R^{T}$
\ENDFOR 
\end{algorithmic}
\end{algorithm}

A key distinction to Adam arises in this substitution: by replacing $g$ with $m_t$ in \cref{eq:deva_linfty}, the numerator of the decoupled variance term becomes $\E[m_t^2]$. This implies that we are accumulating the square of the expected gradient to estimate the second-order moment, rather than the squared sample gradient (Line 6 in \cref{alg:deva_linfty}). This shift also aligns with SNR interpretation 
\citep{orvieto2025search}, where the update magnitude is modulated by the reliability of the gradient signal $m_t$ relative to its stochastic variance $\Var(m_t)$. The complete procedure for $\devalf$ is summarized in \cref{alg:deva_linfty}.

\subsection{Matrix Extension} We now extend our framework to the matrix case, where $X, G\in \R^{n \times m}$. Here we let $H =L^{1/2} \otimes R^{1/2}$, where $L=GG^{T} \in \R^{n \times n}$ and $R=G^{T}G \in \R^{m \times m}$ are the left and right factorized covariance matrices, respectively. For $\Delta X \in \R^{n \times m}$, the DeVA representation yields: 
\begin{equation}
\label{eq:deva_sinfty}
\begin{aligned}
 & \quad \fvec(\Delta X^{*}_{\devasf}) \\
&= -\E[H]^{-1/2} H^{1/2} H^{-1/2} \fvec(G)  \\
&= -
\underbrace{
    \E[(L \otimes R)^{1/2}]^{-1/2}
    (L \otimes R)^{1/4}
}_{\text{Decoupled Variance Adaptation}}
\fvec(\msign(G)).
\end{aligned}
\end{equation}
where the second equality follows $H^{-1/2}\fvec(G)$ gives a $\msign$ update when $H=L^{1/2} \otimes R^{1/2}$. 
Since $\msign(G)$ is the solution of steepest gradient descent under Schatten-$\infty$ ($S_{\infty}$) norm, we name this variant as $\devasf$. While \cref{eq:deva_sinfty} successfully isolates the matrix sign function, the decoupled variance term remains computationally intractable due to the high-dimensional matrix powers involved. To resolve this, we leverage the eigendecomposition approach used in SOAP \citep{vyas2024soap}, allowing us to derive the equivalent but more efficient spectral representation presented in \cref{thm:deva_coordinate_wise}.

\begin{restatable}[Coordinate-wise $\devasf$]{theorem}{devasfde}
\label{thm:deva_coordinate_wise}
Let $L \in \mathbb{R}^{n \times n}$ and $R \in \mathbb{R}^{m \times m}$ be Kronecker factors with eigendecompositions $L = Q_L^{} \Lambda_L^{} Q_L^T$ and $R = Q_R^{} \Lambda_R^{} Q_R^T$. Let $\sigma_i = \sqrt{\lambda_i}$ and $\sigma_j = \sqrt{\mu_j}$ denote the singular values corresponding to the eigenvalues of $L$ and $R$, respectively. With the Kronecker curvature approximation $\mathbb{E}[L^{1/2} \otimes R^{1/2}]$, and assuming local stability (see the proof), the optimal update $\Delta X^{*}$ in \cref{eq:deva} has spectral coordinate-wise form:
\begin{equation}
\label{eq:deva_sinfty_final}
\Delta X^{*}_{\devasf} = -Q_L \left( \widetilde{E}^{-1/2} \odot \msign(Q_L^{T} G Q_R^{}) \right) Q_R^{T},
\end{equation}
where $\widetilde{E} \in \mathbb{R}^{n \times m}$ is the spectral adaptation matrix with entries defined by the ratio of expected to sample gradient singular value products:
\begin{equation}
\widetilde{E}_{ij} = \frac{\mathbb{E}[\sigma_i \sigma_j]}{\sigma_i \sigma_j}.
\end{equation}
Here, $\odot$ and the exponent $-1/2$ denote element-wise operations within the rotated spectral space.
\end{restatable}

\textit{Proof Sketch:}
The derivation proceeds in three main steps:
\begin{enumerate}
    \item \textbf{Spectral Decomposition:} Under the assumption of local stability, the expectation $\mathbb{E}[L^{1/2} \otimes R^{1/2}]$ factorizes as $(Q_L \otimes Q_R) \mathbb{E}[\Lambda_L^{1/2} \otimes \Lambda_R^{1/2}] (Q_L^T \otimes Q_R^T)$.
    
    \item \textbf{Vectorization and Commutation:} Using the identity $(A \otimes B) \operatorname{vec}(X) = \operatorname{vec}(AXB^T)$ and the fact that orthogonal transformations commute with the $\msign$ (i.e., $Q \msign(X) = \msign(QX)$), we project the update into the spectral space defined by $Q_L$ and $Q_R$.
    
    \item \textbf{Diagonal Compression:} We define the diagonal scaling matrix $E = \mathbb{E}[(\Lambda_L \otimes \Lambda_R)^{1/2}] (\Lambda_L \otimes \Lambda_R)^{-1/2}$ and compress the $nm \times nm$ diagonal matrix $E$ into the $n \times m$ matrix $\widetilde{E}$, where $\widetilde{E}_{ij} = \mathbb{E}[\sigma_i \sigma_j] / (\sigma_i \sigma_j)$. 
\end{enumerate}
Substituting these into the vectorized update and mapping back to matrix form yields the Hadamard product $\Delta X^{*} = -Q_L (\widetilde{E}^{-1/2} \odot \msign(Q_L^T G Q_R)) Q_R^T$.
\qed

The coordinate-wise simplification in \cref{thm:deva_coordinate_wise} reveals three significant insights. First, the numerator $\E[\sigma_i \sigma_j]$ tracks the covariance of the singular values of the gradient, rather than the raw gradient covariance $\E[gg^T]$ used in standard vector adaptive methods. Second, the variance adaptation and the matrix sign update occur within the rotated spectral space defined by $Q_L$ and $Q_R$, which provides a principled choice of basis. Finally, the Kronecker curvature approximation can be reduced to an inexpensive element-wise multiplication, facilitating efficient implementations on large-scale problems. The complete derivation is provided in \cref{appsec:proof_deva_co}.   


\subsection{Practical Implementation}
In the $\devasf$ framework, the decoupled variance matrix is composed of entries 
$\tfrac{\E[\sigma_i \sigma_j]}{\sigma_i \sigma_j}$.
However, the numerator requires the expectation of the singular values of the gradient $G$ at each iteration, which is still computationally prohibitive. To address this, we remark on the following connection between the eigenvalues of the factorized covariance matrices and the rotated gradient $G'$.
\begin{restatable}{proposition}{eigprop} \label{prop:eigprop}
For a gradient $G \in \R^{n \times m}$, let $Q_L^{}\Lambda_L^{} Q_L^{T}$ and $Q_R^{}\Lambda_R^{}Q_{R}^T$ be the eigendecompositions of $L=GG^{T}$ and $R=G^TG$, respectively. Let $\lambda_i$ and $\mu_j$ denote the $i$-th and $j$-th eigenvalues of $L$ and $R$. If we define rotated gradient as $G'=Q_{L}^{T}GQ_{R}^{}$, then:
\begin{align}
    \lambda_i &= e_{i}^{T} (G' \odot G') 1_{m}, \quad
    \mu_j=1_{n}^{T}(G' \odot G')e_{j}.
\end{align}
\end{restatable}
\begin{corollary}
\label{cor:deva_adapt_rate}
Let $G'=Q_L^{T}GQ_R^{}$. Let $\sigma_i$ be the $i$-th singular value of $G$ and $\text{rank}(G)=r$, for $i,j \le r $, we have
\begin{align*}
    \sigma_i = \sqrt{\lambda_i} =\norm{G'_{i,\cdot}}_{2}, \text{ or } \sigma_j=\sqrt{\mu_j} = \norm{G'_{\cdot,j}}_{2},
\end{align*}
where $\norm{G'_{i,\cdot}}_{2}$ and $\norm{G'_{\cdot,j}}_{2}$ denote the $\ell_2$-norms of the $i$-th row and $j$-th column of $G'$, respectively.
\end{corollary}
Note that both of these are fancy ways of using the singular value decomposition of $G$, which gives $G'$ as the diagonal matrix of singular values. We state them, however, because they motivate the final update we use as a computable approximation. 
Equipped with \cref{cor:deva_adapt_rate}, we can represent the expectation, $\E[\sigma_i \sigma_j]$, as the entries of the matrix $\E[ rc^{T} ]$, where $r\in\R^{n}$, $c\in\R^{m}$, $r_i=\norm{G'_{i,\cdot}}_{2}$, and $c_j=\norm{G'_{\cdot,j}}_{2}$. 

This structure bears a resemblance to the second-order moment in the idealized Adafactor \citep{shazeer2018adafactor} update discussed in SOAP \citep{vyas2024soap}. However, a key distinction is that idealized SOAP estimates the second momentum by taking the product of independent expectations of squared norm and then normalize, i.e., $(\E[r^2]\E[c^2]^{T})\oslash \E[r^2]$, whereas $\devasf$ take the expectation of the outer product of unsquared norm, i.e., $\E[ rc^{T} ]$. Generally, to compute the preconditioner $M$, the expectation of a Kronecker product in $\devasf$ does not equal the Kronecker product of expectations in SOAP. This discrepancy marks a fundamental distinction between our work and existing methods such as K-FAC \citep{martens2015optimizing} and Shampoo \citep{gupta2018shampoo}.  


To extend $\devasf$ to momentum-based optimization, we apply an EMA to the rotated gradient $G'$ and detail the complete $\devasf$ procedure in \cref{alg:deva_sinfty}. We emphasize that since we accumulate the moving average of the rotated gradient $\ema[G']$ not $G'$ itself, the second momentum should track $\ema[rc^{T}]$ (Line 12--14). Furthermore, following \citet{team2025kimi}, we normalize the orthogonalized update by $0.2 \sqrt{\max(n, m)}$ (Line 16) to align the RMS norm of standard Adam updates. Finally, to maintain computational efficiency, we approximate the full eigendecomposition (Line 8) with a single-step power iteration and QR decomposition, as suggested by \citet{vyas2024soap}.

\section{Analysis}
\label{sec:analysis}

To characterize the updates, we first define a norm that can incorporate the variance adaptation weights.
\begin{definition}[\textbf{$\gamma$-Weighted Dual Norm}]
\label{defn:gamma_norm}
   Let $\|\cdot\|$ be a norm on $\mathbb{R}^n$ and $\langle \cdot, \cdot \rangle$ be the standard inner product. Given a weight vector $\gamma \in \mathbb{R}^n_{++}$, we define the $\gamma$-weighted inner product as $\inner{s}{x}_\gamma := \inner{s\odot\sqrt{\gamma}}{x\odot\sqrt{\gamma}}$. Let $W=\diag(\gamma)$, the $\gamma$-weighted dual norm $\|\cdot\|_{*,\gamma}$ is defined as:
   \begin{align}  
   \norm{s}_{*,\gamma} := \sup_{\norm{x} \leq 1}\inner{s}{x}_\gamma = \norm{Ws}_*
    \end{align}
\end{definition}
Recall that for any $\ell_p$-norm, the solution to the constrained linear minimization problem $\argmin_{\norm{x} \leq 1} \inner{s}{x}$ provides the steepest descent direction. Specifically, for $i=1, \dots, d$, the optimal components are given by $x_i^* = - \sign(s_i)|s_i|^{q-1}/\norm{s}_{q}^{q-1}$ where $1/p + 1/q = 1$. We now provide concrete examples for the norms used in our framework.


\begin{example}[$\ell_\infty$-norm]
    For the $\ell_\infty$-norm, given weights $\gamma_i \ge 0$ for $i=1,\dots,d$, we have
        $\norm{s}_{*,\gamma} = \norm{\gamma \odot s}_1 = \sum_{i=1}^{d}\gamma_i |s_i|$.
\end{example}

\begin{example}[$S_\infty$-norm] 
\label{ex:schatten_infty}
For $X \in \R^{n \times m}$, let $\norm{X}=\sigma_1(X)$ ($S_\infty$-norm), $\norm{S}_1=\Tr(\sqrt{S^{T}S})$,  and $\inner{S}{X}=\Tr(S^{T}X)$. For a positive weight matrix $\Gamma \in \R^{n \times m}_{++}$, we have
$\norm{S}_{*, \Gamma} = \norm{\Gamma \odot S}_1$.
The maximizer is the orthogonal matrix $UV^{T}$ from the SVD of $(\Gamma \odot S)$.
\end{example}

\begin{example}[Preconditioned matrix seminorm \citep{veprikov2025preconditioned}]
For $X \in \R^{n \times m}$, when $\Gamma$ is a diagonal positive matrix $D \in \R^{n \times m}$, preconditioned matrix seminorm $\norm{S}_{*,D}:= \norm{D\odot S}_{*}$ is a special instance of $\Gamma$-weighted dual norm. 
\end{example}
\begin{example}[Nuclear rank of $S$ \citep{davis2025spectral}]
\label{ex:nuclear_rank}
For $X \in \R^{n \times m}$, when $\Gamma$ is a diagonal positive matrix $D \in \R^{n \times m}$ where $D_{ii}=1/\norm{S}_F$ for all $i$, $\norm{S}_{*,D}^2=\norm{S}_{*}^2/\norm{S}_F^2$ is the nuclear rank of $S$. 
\end{example}
\begin{remark}
When $S=\nabla f(X)$, the nuclear rank measurement (\cref{ex:nuclear_rank}) serves as a theoretical indicator for when spectral descent outperforms standard gradient descent \citep{davis2025spectral, shen2025convergence}.
\end{remark}

\begin{remark}
\label{rem:gamma_norm_bdd}
In the $S_\infty$-norm  case of \cref{ex:schatten_infty}, we have   $\min(\Gamma) \norm{S}_{1}\le\inner{S}{\msign(S)}_{\Gamma} \le  \inner{S}{\msign(\Gamma \odot S)}_{\Gamma}= \norm{S}_{1, \Gamma}$.
\end{remark}
\cref{rem:gamma_norm_bdd} justifies that we can consider $\inner{S}{\msign(\Gamma \odot S)}_{\Gamma}$ as a valid measure for the stationarity condition of the matrix-valued function in \cref{thm:devasf}.
\begin{figure*}[!t]
    \centering
    \includegraphics[width=\linewidth]{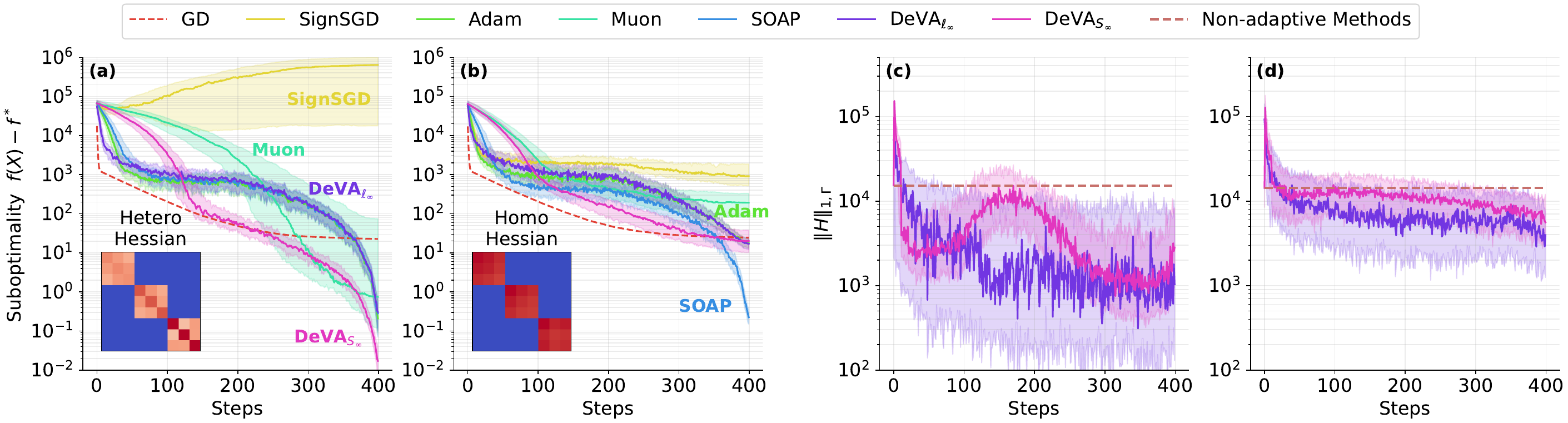}
    \caption{Trace Quadratic Function Optimization (\textit{median} and \textit{25\%/75\% quantiles} over 100 seeds). \textbf{(a)--(b)}: Training performance for various optimizers on 9-dimensional trace quadratic problems with homogeneous vs. heterogeneous Hessians (details see \cref{appsec:trace}). $\devasf$ significantly outperforms Muon in the heterogeneous setting. \textbf{(c)--(d)}: Dynamic of weighted dual norm $\norm{H}_{1,\Gamma}$ (see \cref{defn:gamma_norm}) for adaptive methods. Adaptive methods ($\devalf$, $\devasf$) effectively reduce $\|H\|_{1,\Gamma}$ as predicted by Theorems \ref{thm:devalf} and \ref{thm:devasf}, with a more pronounced reduction in the heterogeneous case \textbf{(c)} compared to the homogeneous case \textbf{(d)}.}
     \label{fig:trace-func-avg-plot}
\end{figure*}

\subsection{Convergence result for $\devalf$}
To facilitate analysis, we consider a simplified version of the \emph{vector}-case of  $\text{DeVA}_{\ell_\infty}$ where the momentum terms are zero ($\beta_1 = \beta_2 = 0$), resulting in the following update rule:
\begin{align}
\label{eq:devalf_simplified}
x_{t+1} = x_t - \eta \gamma_t \odot \text{sign}(g_t).
\end{align}
We first define the smoothness geometry below.
\begin{assumption}[Blockwise Smoothness]
\label[assumption]{ass:vec_lip}
The objective function $f: \R^d \to \R$ is blockwise $L_v$-smooth. Specifically, there exists a partition of the $d$ dimensions into $c$ blocks and a smoothness vector $L_v = [L_{v,1}, \dots, L_{v,c}] \in \mathbb{R}^c_{++}$ such that for any $x, y \in \mathbb{R}^d$ we have $f(y) \le f(x) + \inner{\nabla f(x)}{y-x} + \frac{1}{2}\sum_{i=1}^{c}L_{v,i} \norm{y_i-x_i}_2^2$, where $x_{i}$ is the $i$-th block.
\end{assumption}

    The blockwise smoothness assumption addresses the blockwise diagonal structure in Hessian frequently observed in neural networks \citep{ghorbani2019investigation,zhang2024adam}. This Hessian heterogeneity means that the standard choice of step size—namely, the inverse of a single global smoothness constant used in convex smooth analysis \citep{nesterov2018lectures}—is no longer optimal, motivating instead a blockwise step size that adapts to the local curvature of each block.
\begin{remark}
For notation consistency, we broadcast $L_{v,i}$ to each block to match the dimension of $\R^d$ in the later analysis.
\end{remark}

\begin{assumption}
\label[assumption]{ass:vec_grad_noise}
    Given mini-batch stochasticity $\xi$ and stochastic gradient $g := \nabla f(x,\xi) \in \R^d$, we assume an unbiased estimator $\E[g] = \nabla f(x)$. Furthermore, for any $i=1,\dots,d$, the variance is bounded such that $\E\left[\big(g_i - \nabla f(x_i)\big)^2\right] \le \sigma_i^2$.
\end{assumption}

\begin{remark}
In our analysis, we utilize a mini-batch stochastic gradient $g_t = \frac{1}{n_t} \sum_{j=1}^{n_t} \nabla f(x_t, \xi_j)$. This reduces the effective variance per block to $\sigma_i^2/n_t$. We use $\gF_{t}:=\sigma(\xi_1,\dots,\xi_t)$ to denote the natural filteration.
\end{remark}

\begin{restatable}{theorem}{vecconv}\label{thm:devalf}
    Under \cref{ass:vec_lip} and \cref{ass:vec_grad_noise}, for $\sigma_v ,L_v , \gamma_t \in \R^d$,  let $\Delta = f(x_1) - f^*$, for a fixed learning rate $\eta = 1/\sqrt{T}$, a mini-batch size $n_t = T$, suppose the weight process $\{\gamma_t\}^{d}_{>0}$ is bounded and $\gF_{t-1}$ measurable,  let $\gammahat = \frac{1}{T}\sum_{t=1}^{T}\gamma_t$ and $\gammahatsq = \frac{1}{T}\sum_{t=1}^{T}\gamma_t^2$, the iterates using \cref{eq:devalf_simplified} satisfy:
\begin{align*}
    \frac{1}{T}\sum_{t=1}^{T} \norm{\nabla f(x_t)}_{1,\gamma_t} \le \frac{1}{\sqrt{T}} \Big (\Delta + 2\norm{\sigma_v}_{1,\gammahat} + \frac{1}{2} \norm{L_v}_{1, \gammahatsq}\Big).
\end{align*}
\end{restatable}
Theorem \ref{thm:devalf} generalizes the non-adaptive guarantees of \citep{bernstein2018signsgd} by replacing the standard smoothness error $\norm{L_v}_1$ with a weighted term $\norm{L_v}_{1, \hat{\gamma}^2}$. This new metric quantifies the alignment between blockwise smoothness and our curvature estimate, where $\hat{\gamma}_i^2$ corresponds to the $i$-th diagonal element of the inverse Fisher information approximation, $\mathbb{E}[\text{diag}(gg^T)]^{-1}$.

Unlike standard SignSGD (where $\gamma_t = \mathbf{1}_d$), our decoupled adaptive approach uses a weight process $\gamma_t \propto \mathbb{E}[\text{diag}(g_t g_t^T)]^{-1/2}$ to actively scale the smoothness parameters. In the case of $\hat{\gamma}_i^2 \le 1$, the error term $\hat{\gamma}_i^2 L_{v,i}$ effectively reduces the impact of blockwise smoothness, potentially lowering the total error $\norm{L_v}_{1, \hat{\gamma}^2}$ toward the dimension $d$ in optimal cases. Full proofs are available in Appendix \ref{appsec:proof_devalf}.

\subsection{Convergence result for $\devasf$}
Now turn our attention to the \emph{matrix case} and we consider the same type of simplified version of the $\devasf$ update rule with zero momentum ($\beta_1 = \beta_2 = \beta_3 = 0$):
\begin{align}
\label{eq:devasf_simplified}
    X_{t+1} = X_{t} - \eta \Gamma_t \odot \msign(G_t).
\end{align}
To analyze this matrix update, we extend the concept of blockwise smoothness to the matrix manifold.

\begin{assumption}
\label[assumption]{ass:matrix_lip}
For any matrices $X, Y \in \R^{n \times m}$, there exists a smoothness matrix $L_m \in \R^{n \times m}_{++}$ such that $f(Y) \le f(X) + \inner{\nabla f(X)}{Y-X} + \frac{1}{2}\norm{\sqrt{L_m} \odot (Y-X)}_{F}^2$, where $\norm{\cdot}_F$ denotes the Frobenius norm and $\sqrt{L_m}$ is the element-wise square root of $L_m$.
\end{assumption}
\begin{remark}
\cref{ass:matrix_lip} can be viewed as a specific instance of the blockwise smoothness in \cref{ass:vec_lip}, where each entry of the matrix is treated as an individual block ($c = nm$).
\end{remark}
\begin{restatable}{theorem}{matconv}
\label{thm:devasf}
    Under \cref{ass:matrix_lip} and \cref{ass:vec_grad_noise}, for $\sigma_m, L_m, \Gamma_t \in \R^{n \times m}$, let $r=\min(n,m)$ and $\Delta = f(X_1) -f^{*}$, for a learning rate $\eta = 1/\sqrt{T}$ and mini-batch size $n_t = T$, suppose the weight process $\{\Gamma_t\} \subset \R^{n \times m}_{>0}$ is bounded and $\gF_{t-1}$ measurable, let $\Gammahat=\sqrt{\frac{1}{T}\sum_{t=1}^{T}\Gamma_t^2}$ and $\Gammahatsq=\frac{1}{T}\sum_{t=1}^{T}\Gamma_t^2$. Defining the stationarity measure $\psi(G_t) = \E[\inner{G_t}{\msign (G_t)}_{\Gamma_t}]$, the iterates using \cref{eq:devasf_simplified} satisfy:
\begin{align*}
    \frac{1}{T}\sum_{i=1}^{T} \psi(G_t) \le \frac{1}{\sqrt{T}} \Big (\Delta + \sqrt{r}\norm{\sigma_m}_{F, \Gammahat} + \frac{1}{2} \norm{L_m}_{\ell_1,\Gammahatsq}\Big).
\end{align*}
\end{restatable}

\cref{thm:devasf} provides a novel convergence guarantee for adaptive matrix-sign methods, revealing a controllable blockwise smoothness $\norm{L_m}_{\ell_1,\Gammahatsq}$ in the final bound, mirroring the result of the vector case in \cref{thm:devalf}. 

\begin{figure}[!t]
    \centering
    \includegraphics[width=0.94\linewidth]{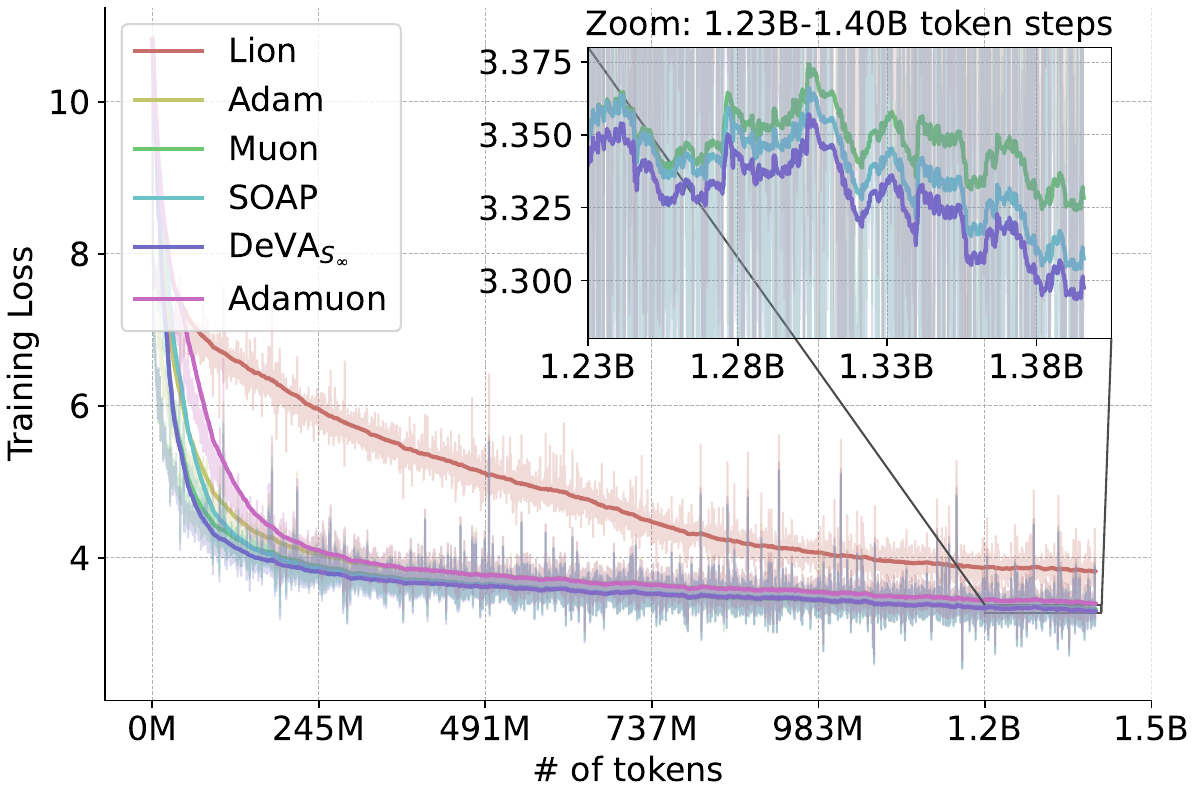}
    \caption{EMA-smoothed training loss for a 274M NanoGPT model on FineWeb. All methods use a uniform learning rate of $0.001$. $\devasf$ consistently outperforms baseline optimizers throughout the training duration.}
    \label{fig:nano_gpt_loss}
\end{figure}

If we remove the adaptive step size ($\Gamma_t = 1_{n \times m}$), the stationarity measure reduces to the nuclear norm $\psi(G_t) = \E[\norm{G_t}_1]$, which is an upper bound on the nuclear norm of the true gradient, i.e., $\E[\norm{\nabla f(X_t)}_1] \le \E[\norm{G_t}_{1}]$. This leads to the following result for non-adaptive spectral descent:

\begin{corollary}
\label{cor:devasf}
Under the settings of Theorem \ref{thm:devasf}, for $\sigma_m, L_m\in \R^{n \times m}$, if $\Gamma_t = 1_{n \times m} $ for all $t$, let $\eta =\sqrt{\frac{2\Delta}{T \norm{L_{m}}_{\ell_1}}}$ and $n_t=\frac{r \norm{\sigma_m}_F^2}{2\Delta \norm{L_m}_{\ell_1}}T$, we have
\begin{align*}
    \frac{1}{T}\sum_{i=1}^{T} \E[\norm{\nabla f(X_t)}_1] \le \frac{2}{\sqrt{T}} \Big ( \sqrt{2\Delta} \cdot \sqrt{ \norm{L_m}_{\ell_1}}\Big).
\end{align*}
\end{corollary}
This corollary shows that without variance adaptation, spectral descent achieves a batch complexity of $O( \Delta  \norm{L_m}_{\ell_1} \epsilon^{-2})$ to reach $\epsilon$-stationarity. Unlike $n_t=1$ and a non-zero first-order momentum coefficient $\beta_1$ in \citep{shen2025convergence}, a large batch size $n_t =O(r\norm{\sigma_m}_{F}^2\epsilon^{-2})$ is required in our case and the total sample complexity will be $O(r \Delta \norm{\sigma_m}_{F}^2 \norm{L_m}_{\ell_{1}} \epsilon^{-4})$. This result matches the recent complexity established for Muon in \citep[Corollary 4.2]{shen2025convergence}, while being more general in explaining how decoupled variance adaptation accelerates the convergence rate. The detailed proof is provided in Appendix \ref{appsec:proof_devasf}.

\section{Experiment}
\label{sec:experiment}
Our implementation for computing the spectral-norm constrained update follows the Newton-Schultz iteration proposed in \citep{jordan2024muon}. In addition, the matrix sign function is multiplied by an RMS alignment in \citep{team2025kimi} for both Muon and $\devasf$. All the following experiments are conducted on a single 40GB A100 GPU.

\begin{figure}[!t]
    \centering
    \includegraphics[width=0.9\linewidth]{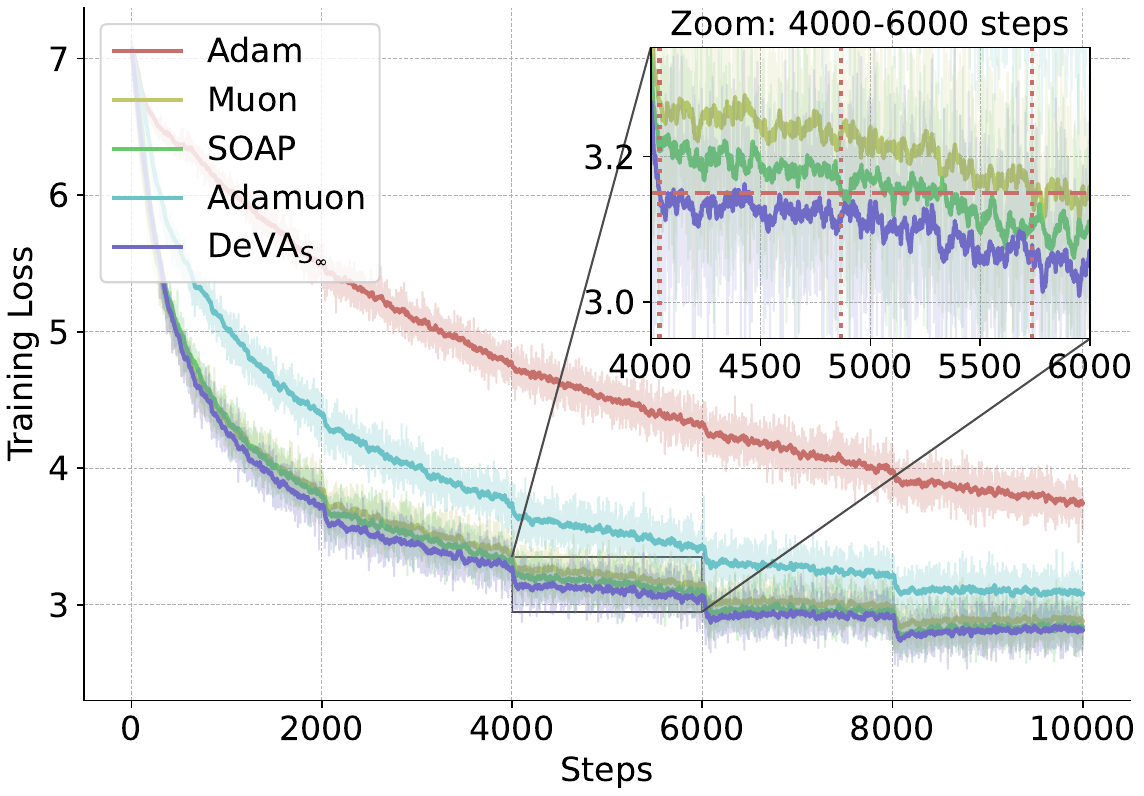}
    \caption{EMA-smoothed training loss for ViT-L/16 on ImageNet-1K during the first five epochs. All methods utilize a fixed learning rate of $0.001$ without decay. $\devasf$ demonstrates significantly faster initial convergence compared to other optimizers.}
    \label{fig:imagenet_loss}
\end{figure}

\subsection{Main Results}

\paragraph{Quadratic Trace Function Optimization.} 
Neural network Hessians typically exhibit a block-diagonal structure \cite{ghorbani2019investigation}. Recent work \citep{zhang2024adam, orvieto2025search} has utilized quadratic functions $f(x) =\frac{1}{2}x^{T}Hx$ with heterogeneous eigenspectra to illustrate why adaptive approaches are necessary in such settings. We extend this analysis to the matrix domain by considering the minimization of a quadratic trace function: $f(X)=\frac{1}{2}\Tr(X^{T}HX)$ where $X, H \in \R^{d \times d}$ and $H$ is a fixed matrix with a heterogeneous/homogeneous eigenspectrum. This simple optimization problem is of particular interest, as In-Context Learning (ICL) \citep{garg2022can} of linear transformer can be framed as a special instance of the quadratic trace problem \citep{ma2026preconditioning}. Detailed notes for this example  are provided in Appendix \ref{appsec:trace}.

As shown in \cref{fig:trace-func-avg-plot}, deterministic gradient descent exhibits identical convergence rates across both homogeneous and heterogeneous functions. In contrast, the performance gain of $\devasf$ over the non-adaptive Muon mirrors the advantage of adaptive vector methods ($\devalf$, Adam) over the SignSGD. Notably, SOAP demonstrates similar convergence behaviors on both function types. We conjecture this occurs because SOAP lacks an explicit decoupling between the matrix sign update and the adaptive step size; consequently, it avoids the stagnant convergence patterns often observed in pure sign-based optimizers. Finally, \cref{fig:trace-func-avg-plot} reveals a significant reduction in the smoothness-related error $\norm{H}_{1,\Gamma}$ when the function is Hessian-heterogeneous. This empirical reduction validates our theoretical intuition and underscores the importance of adaptive spectral methods for large-scale neural network training.

\paragraph{GPT.} We evaluate our methods using the modded-nanogpt codebase \citep{modded_nanogpt_2024}, which incorporates modern architectural enhancements: rotary embeddings, RMSNorm, a linear decay schedule, and $\text{ReLU}^2$ activations. The snapshot we used includes three additional value embedding layers following the initial transformer blocks and an untied output head, bringing the total parameter count to 274M. We pretrain on the FineWeb-Edu dataset \citep{penedo2024fineweb}; detailed hyperparameter sweeping experiments over learning rate and shampoo's $\beta$ are available in Appendix \ref{appsec:modded_nanogpt}. As shown in \cref{fig:nano_gpt_loss} and \cref{tab:nanogpt}, $\devasf$ consistently outperforms both Muon and SOAP in validation perplexity.

\begin{figure}[t]
    \centering
    \includegraphics[width=0.95\linewidth]{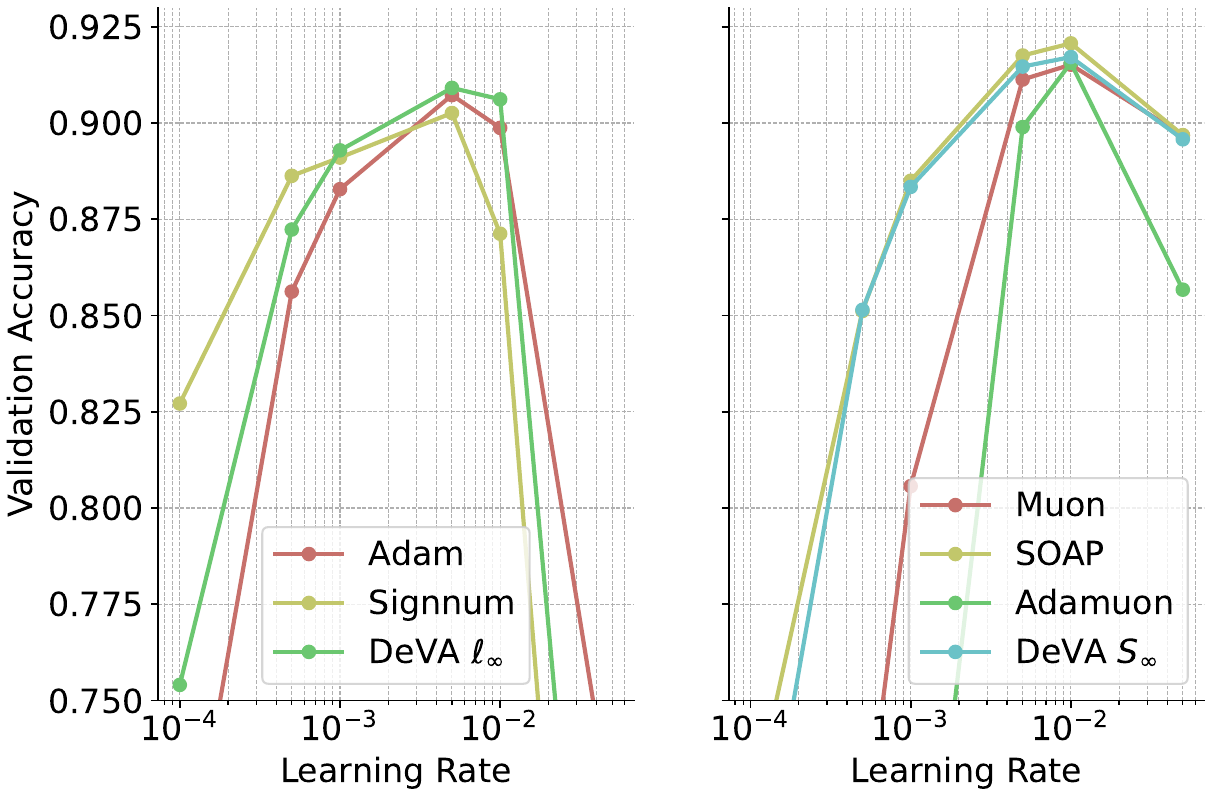}
    \caption{ResNet-20 validation accuracy on CIFAR-10 after 40 epochs. Optimal learning rates cluster near $0.005$ for vector methods (\textbf{left}) and $0.01$ for matrix methods (\textbf{right}), where $\devalf$ and SOAP achieve peak performance, respectively.}
    \label{fig:cifar10_optimal_lrs}
\end{figure}

\begin{table}
    \centering
       \caption{Final validation perplexity and training runtime (ms/1k tokens) for a 270M NanoGPT model pretrained on 1.4B tokens. Values in $(\pm\%)$ indicate relative gain or loss compared to Muon.}
    \begin{tabularx}{0.9\linewidth}{llXX}
    \toprule 
         & Muon & SOAP & $\devasf$ \\  \midrule 
    Loss & 3.305 & 3.304 \textcolor{gray}{(-0.0\%)} & 3.271 \textcolor{gray}{(-1\%)} \\
    Time & 10.72s & 15.79s \textcolor{gray}{(+47\%)} & 16.49s \textcolor{gray}{(+54\%)} \\ \bottomrule
    \end{tabularx}
    \label{tab:nanogpt}
\end{table}

\paragraph{Image Classification.}
We evaluate performance on vision transformers (ViT-L/16 \citep{dosovitskiy2020image}) using ImageNet-1K and convolutional neural networks (CNNs) on CIFAR-10.  ImageNet-1K training trajectories are illustrated in \cref{fig:imagenet_loss}.  As shown in \cref{fig:cifar10_optimal_lrs}, following RMS alignment \citep{team2025kimi}, the optimal learning rate consistently transfers across both vector and matrix optimizers, stabilizing between $0.005$ and $0.01$. For exhaustive experimental configurations, refer to Appendices \ref{appsec:cifar-10} and \ref{appsec:imagenet-1k}.

\section*{Acknowledgements}
We thank Yufan Huang for proofreading the manuscript. DFG would like to acknowledge DOE DE-SC0023162 Sparsitute MMICC center for partial support as well as NSF Nonlinear graph IIS-2007481. ZZ would like to acknowledge support from N000142412621, Office of Naval Research. CSB and BB would like to acknowledge support from NSF IIS-2420724.

\section*{Impact Statement}

This paper presents work whose goal is to advance the field of Machine
Learning. There are many potential societal consequences of our work, none
which we feel must be specifically highlighted here.


\bibliography{example_paper}
\bibliographystyle{icml2026}

\newpage
\appendix
\onecolumn
\begin{center}
    {\large \bf Appendix} 
    \vspace{1em} 
\end{center}

\addtocontents{toc}{\protect\setcounter{tocdepth}{2}}
\tableofcontents

\newpage

\section{Related Work}
\label{appsec:related_work}

In this section, we provide a structured overview of the optimization landscape relevant to our work, focusing on how different geometric perspectives influence convergence and efficiency in deep learning. We first review the history of adapting vector-based ideas to matrix variables in the domain of semi-definite programming.  We next examine Gradient Descent under Non-Euclidean Norms, tracing the evolution from coordinate-wise sign methods to generalized $\ell_p$ geometries. We then discuss the emerging paradigm of Spectral Geometry and Modular Norms, which utilizes operator-norm updates to better align with the architectural properties of large-scale models. Finally, we survey Curvature-aware Methods, reviewing how local Hessian and Fisher information are efficiently approximated to navigate the complex loss surfaces of modern neural networks.

\paragraph{From vectors in linear programming to matrices in semi-definite programming}
Machine learning was a key driver of advanced semi-definite programming through the early 2000s. 
Semi-definite programming arises when a linear objective on a vector of variables is generalized to a linear objective on a matrix of variables, with the vector non-negativity constraints replaced by positive semi-definiteness of the matrix~\citep{Boyd2004}. 
Many SDPs in machine learning arose as \emph{convex} relaxations of combinatorial objectives, following the general paradigm exemplified by classical relaxations such as MaxCut~\citep{Goemans1995}. 
Famous examples of SDPs in ML include matrix completion with nuclear norms~\citep{Candes2009,Recht2010}, which helped popularize the nuclear norm as the matrix-based analogue of the $1$-norm in compressed sensing and sparse optimization. 
Related SDP formulations also appeared in manifold learning, embedding, and clustering problems~\citep{Weinberger,pmlr-v2-kulis07a}, motivated advanced algorithms for regularized graph diffusions~\citep{Orecchia-2011-implicit-regularization}, and more recently lower-bound on graph computations~\citep{Huang-2023-mucond-lrsdp}. A key modeling step in much of this research was adapting vector-based ideas to matrix variables, and our work studies descent algorithms within this framework.

\paragraph{Gradient Descent under non-Euclidean norm}
Starting from SignSGD \citep{bernstein2018signsgd}, a significant line of research has studied sign-based methods, which can be viewed as (stochastic) steepest descent under the $\ell_{\infty}$ norm. This extends the classical $\ell_{2}$-based regime used in vanilla SGD \citep{ghadimi2013stochastic}. By scaling down sparse
noise and scaling up dense gradients comparatively, sign updates benefit learning when the gradients are dense but contain a sparse set of extremely noisy components \citep{bernstein2018signsgd}. 

Recent work on Lion \citep{chen2023symbolic} and its generalization Lion-$\mathcal{K}$ \citep{chen2023lion} further underscores the empirical benefits of sign-driven coordinates in large-scale tasks. More recently, \citet{luo2025stacey} interpolate between the extremes of non-Euclidean geometry ($\ell_2$ and $\ell_\infty$) by considering general $\ell_p$ norms for $2 < p < \infty$, through which a more favorable Lipschitz constant $L_p$ can be leveraged in high dimensions.

\paragraph{Spectral Geometry and Modular Norms}
Aside from $\ell_p$ norms, there is an increasing trend in research toward using spectral geometry for training deep learning models. In the Schatten-$\infty$ norm, the training loss of RBMs and CNNs admits a tighter majorization bound than in the Frobenius norm, motivating steepest descent in the spectral norm rather than the Euclidean norm \citep{carlson2015stochastic, carlson2015preconditioned}.

Muon \citep{jordan2024muon} applies a similar spectral update—implemented with a Newton–Schulz approximation to the polar factor—to the hidden layers of modern architectures and has shown strong empirical performance for NanoGPT-scale models and other large-scale training tasks \citep{wen2025fantastic, liu2025muon}. The Scion optimizer \citep{pethick2025training} builds a general framework of stochastic conditional gradient (LMO-based) methods with norm constraints; within this framework, Muon is a special case, while Scion allows for the use of operator norms (e.g., RMS, spectral, and $\ell_{\infty}$ norms) across different layers.

Building upon spectral and $\ell_{\infty}$ norms, the new concept of the modular norm \citep{bernstein2024modular, bernstein2024old, large2024scalable} has been proposed to unify weights with different geometries in large-scale networks and popularize spectral and operator-norm updates. Recent research has also shown that spectral descent can outperform vanilla gradient descent when the matrix gradient has a high nuclear-to-Frobenius ratio \citep{shen2025convergence, davis2025spectral}, aligning with the dense gradient analysis for vector gradients in SignSGD \citep{bernstein2018signsgd}.

\paragraph{Curvature-aware Methods}

In parallel, a variety of optimizers capture the intricate geometry of the loss surface by coupling parameter updates through local curvature information. Updates leveraging second-order information can achieve significantly more progress per iteration than simple scaled gradients. At the convex settings, preconditioned gradient descent methods \citep{gonen2016solving, frangella2024promise} were among the first to utilize this geometric information. 
Subsequent Hessian estimation methods \citep{byrd2016stochastic, roosta2019sub, frangella2024sketchysgd} maintain stochastic online approximations of the Hessian for neural networks. More recently, \citet{depavia2025faster} showed that adaptive optimization methods can be accelerated via a change-of-basis technique named EGOP reparameterization that reduces the stable rank of a convex objective on deterministic optimization. For nonconvex matrix factorization problems, preconditioning has further emerged as a powerful accelerated optimization tool for both exactly and overly parameterized settings \citep{tong2021accelerating, zhang2023preconditioned}. However, updates involving local curvature can be computationally prohibitive, as they typically require matrix inversions or matrix-vector products with the curvature matrix. Consequently, there is a strong demand for efficient methods that directly compute the inverse of high-quality curvature approximations.

Notable examples include K-FAC \citep{martens2015optimizing} and Shampoo \citep{gupta2018shampoo}, which approximate the Fisher information matrix via Kronecker products, and SOAP \citep{vyas2024soap}, which approximates the Hessian using an optimal Kronecker product. To further mitigate the computational overhead, ASGO \citep{an2025asgo} only keeps a single-side preconditioner based on Shampoo's preconditioner. To go beyond that, Adam \citep{adam2014method} maintains only the diagonal elements of the Gauss-Newton matrix.

To explain the success of such diagonal approximations, extensive prior work has shown that the Hessian of a neural network tends to exhibit a block-diagonal structure \citep{ghorbani2019investigation, zhang2024transformers, zhang2024adam}, where each block corresponds to an individual layer or neuron. Furthermore, numerous studies \citep{wu2020dissecting, yao2020pyhessian} have observed that neural network Hessians are typically low-rank; their effective rank—measured by the ratio $\|H\|_{*} / \|H\|_{op}$—is often significantly smaller than their ambient dimensionality.

To further accelerate spectral descent, particularly for optimizers like Muon, recent methods have begun incorporating second-order information. Optimizers such as Adamuon \citep{si2025adamuon}, Normuon \citep{li2025normuon}, AdaGO \citep{zhang2025adagrad}, and COSMOS \citep{liu2025cosmos} integrate Adam-style adaptive stepsizes into the Muon update. Concurrent work \citep{eschenhagen2026clarifying} also shows that Shampoo achieves higher token efficiency than Muon, mirroring Adam's advantage over Signum. However, the fundamental question regarding the theoretical origin and justification of second-order stepsizes within spectral gradient descent frameworks remains an open challenge.
\newpage
\section{Trace Quadratic Example}
\label{appsec:trace}
Our setup here is inspired directly from the results and discussions in \citep{zhang2024adam,orvieto2025search}. In particular, we consider the problem of minimizing the following function $f : \R^{d \times d} \to \R$:
\begin{align*}
    f(X) = \frac{1}{2} \Tr (X^TH X).
\end{align*}
It is not to difficult to verify that $\nabla f(X)=\frac{1}{2}(H+H^{T})X$ and $\nabla^2 f(X) = I_{d \times d} \otimes \frac{1}{2}(H+ H^T)$. Thus, in order to construct a Heterogeneous (Homogeneous) Hessian for $f(X)$, we just need to construct a Heterogeneous (Homogeneous) $H_{hom}$ and $H_{het}$. 

Specifically, we fix the eigenvalues of  $H_{hom}$ and $H_{het}$, i.e., $\eig (H_{hom})=\eig (H_{het})=\{1,2,3,99,100,101, 4998,4999,5000\}$. We choose both Hessians to be block-diagonal, with 3 by 3 blocks size. The homogeneous $H_{hom}$ has eigenvalues of different magnitudes in each block, while the Heterogeneous $H_{her}$ keeps similar magnitudes in each block, i.e., 
\begin{align*}
    H_{hom} &= [[1,2,3], [99,100,101], [4998,4999,5000]] \\
    H_{het} &= [[1,99,4998],[2,100,4999],[3,101,5000]].
\end{align*}
To add randomness, we generate a 3 by 3 random rotation matrix $Q_{3 \times 3}$ to the block matrix of $H_{het}$ and $H_{hom}$.  Each rotation is sampled from the decomposition of a random 3 by 3 positive semidefinite matrix $AA^{T}$, where $A \in \R^{3 \times 3}$ has i.i.d. Gaussian entries.

One may notice that $H_{hom}$ is indeed a permuted version of $H_{het}$, thus, running deterministic gradient descent on the two variants should show no discrepancy. However, under the stochastic case, the empirical learning process between the two variants demonstrated a significant difference \citep{zhang2024adam,orvieto2025search}. To introduce stochasticity in the gradient, we follow \citep{orvieto2025search} to take the square root of the Hessian to define a 9 by 9 design $A$, i.e., $H=A^TA$. In practice, we use Cholesky decomposition for $H$ to obtain $A$ efficiently. Then the target function can be rewritten as
\begin{align*}
    f(X)=\frac{1}{2}\Tr (X^{T}A^TAX) = \frac{1}{2} \norm{AX}_F^2.
\end{align*}
Write row of $A$ as $a_i^{T}$, then
\begin{align*}
    f(X) =\frac{1}{2} \sum_{i=1}^d \norm{a_i^{T} X}_2^2,
\end{align*}
each row is like one ``sample'' and we can use the Kaczmarz randomized algorithm to sample the gradient of $f$.

In practice, all optimizers in \cref{fig:trace-func-avg-plot} except Gradient descent use moving average parameters $\beta_1$ set to zero and $\beta_2$ set to $0.99$ to emphasize the effect of second-order adaptation. A learning rate scheduler with 50\% warmup followed up by a linear decay to zero is implemented for optimizers except gradient descent.

\newpage
\section{Proofs for \cref{sec:methods}}
\subsection{Proof for \cref{thm:deva_coordinate_wise}}
\label{appsec:proof_deva_co}
\devasfde*
\begin{proof}
Recall when $H=L^{1/2} \otimes R^{1/2}$, DeVA gives us
\begin{align}
    \fvec (\Delta X^{*}) 
    &= -\E[(L^{1/2} \otimes R^{1/2})]^{-1/2} (L^{1/2} \otimes R^{1/2})^{1/2} \cdot  (L^{1/2} \otimes R^{1/2})^{-1/2} \fvec (G) \\
    &= -\E[(L \otimes R) ^{1/2}]^{-1/2} (L \otimes R)^{1/4}\fvec(\msign (G))
\end{align}
Let $\mathcal{K} = \Lambda_L \otimes \Lambda_R$ be the diagonal matrix of Kronecker eigenvalues. Substituting the eigendecompositions $L = Q_L^{} \Lambda_L^{} Q_L^T$ and $R = Q_R^{} \Lambda_R^{} Q_R^T$ into the expectation, we have:
\begin{equation}
\mathbb{E}[L^{1/2} \otimes R^{1/2}] = (Q_L \otimes Q_R) \mathbb{E}[\mathcal{K}^{1/2}] (Q_L^T \otimes Q_R^T),
\end{equation}
where we apply our local stability assumption that the eigenbases $Q_L, Q_R$ are ``locally stable'' such that the expectation acts only on the eigenvalues. The vectorized update $\text{vec}(\Delta X^*)$ is given by:
\begin{equation}
\text{vec}(\Delta X^*) = -(Q_L \otimes Q_R) \mathbb{E}[\mathcal{K}^{1/2}]^{-1/2} (\mathcal{K}^{1/2})^{1/2} (Q_L^T \otimes Q_R^T) \text{vec}(\msign(G)),
\end{equation}
Using the property that for orthogonal $Q$, $\msign(Q^T G Q) = Q^T \msign(G) Q$, and applying the Kronecker identity $(A \otimes B)\text{vec}(X) = \text{vec}(AXB^T)$, we write:
\begin{equation}
\text{vec}(\Delta X^*) = -(Q_L \otimes Q_R) E^{-1/2} \text{vec}( \msign(Q_L^TGQ_R^{}) ),
\end{equation}
where $E = \mathbb{E}[\mathcal{K}^{1/2}] \mathcal{K}^{-1/2}$ is a diagonal matrix of size $nm \times nm$.

To simplify the operation, we observe that the $k$-th diagonal entry of $E$, corresponding to the index pair $(i, j)$, is:
\begin{equation}
E_{kk} = \frac{\mathbb{E}[\sqrt{\lambda_i \mu_j}]}{\sqrt{\lambda_i \mu_j}}.
\end{equation}
Let $\sigma_i = \sqrt{\lambda_i}$ and $\sigma_j = \sqrt{\mu_j}$ denote the singular values of the gradient factors. We can represent the diagonal action of $E^{-1}$ on the vectorized matrix as an element-wise product (Hadamard product) in matrix form. We define the spectral adaptation matrix $\widetilde{E} \in \mathbb{R}^{n \times m}$ such that $\widetilde{E}_{ij} = E_{kk}$. Specifically:
\begin{equation}
\widetilde{E}_{ij} = \frac{\mathbb{E}[\sigma_i \sigma_j]}{\sigma_i \sigma_j}.
\end{equation}
Substituting this back into the matrix form of the update rule and apply the identity $(A \otimes B)\text{vec}(X) = \text{vec}(AXB^T)$ again yields:
\begin{equation}
\Delta X^* = -Q_L \left( \widetilde{E}^{-1/2} \odot (Q_L^T \msign(G) Q_R) \right) Q_R^T.
\end{equation}
This completes the proof.

\end{proof}

\subsection{Proof for \cref{prop:eigprop}}
\eigprop*
\begin{proof}
Let $Q_L = [u_1,\dots, u_n]$ and $Q_R = [v_1,\dots,v_m]$. The rotated gradient is given by $G'=Q_L^{T}GQ_{R}$, with entries $G'_{ij}=u_i^{T}Gv_j$. Thus $e_{i}^{T} (G' \odot G') 1_{m} =\sum_{k=1}^{m} (G'_{ik})^2=\sum_{k=1}^n(u_i^{T}GQ_Re_{k})^2=\norm{u_{i}^{T}GQ_R}^2_2=u_iGG^{T}u_i^{T}=\lambda_i$. The derivation for $\mu_j$ follows symmetrically.
\end{proof}

\section{Proofs for \cref{sec:analysis} }
\paragraph{Setup and notation.} Let $\{\xi\}_{t \ge 1}$ denote the sequence of i.i.d. mini-batch indices drawn during training wich size of $n_t$. Define the natural filtration
\begin{align*}
    \gF_0 : = \{\emptyset, \Omega\}, \quad \gF_{t}:=\sigma(\xi_1,\xi_2,\dots, \xi_t) \quad \text{for } t\ge 1.
\end{align*}
\subsection{Proof for \cref{thm:devalf}}
\label{appsec:proof_devalf}
\vecconv*
\begin{proof}
Using blockwise smooth assumption in \cref{ass:vec_lip} and broadcast $L_{v,i}$ in each block to form a $\R^d$ vector yields the following descent inequality
\begin{align*}
    f(x_{t+1}) &\le f(x_t) + \inner{\nabla f(x_t)}{ x_{t+1}-x_t} + \frac{1}{2} \sum_{i=1}^c L_{v,i} \norm{x_{t+1,i}-x_{t,i}}_2^2 \\
    &\le f(x_t) -\eta \sum_{i=1}^d  \gamma_{t,i} \nabla_i f(x_t)  \sign(g_t)_i + \frac{\eta^2}{2}\sum_{i=1}^{c}L_{v,i} \sum_{j\in \gB_i} \gamma_{t,j}^2. 
\end{align*}
Since we broadcast $L_{v,i}$ for all $j \in \gB_i$, we can push $L_{v,i}$ inside the inner sum, then the above inequality equivalents to 
\begin{align*}
    f(x_{t+1}) &\le f(x_t) -\eta \sum_{i=1}^d  \gamma_{t,i} \nabla_i f(x_t)  \sign(g_t)_i + \frac{\eta^2}{2} \sum_{i=1}^d L_{v,i} \gamma_{t,i}^2,
\end{align*}
using the sign mismatching trick used in \citep{bernstein2018signsgd}, we know $\nabla_i f(x_t)\sign(g_t)_i = |\nabla_i f(x_t)| - 2 |\nabla_i f(x_t) | \mathbb{I}[\sign(\nabla_i f(x_t)) \ne \sign(g_t)_i]$, plug this back into the above descent inequality yields:
\begin{align*}
    f(x_{t+1}) \le f(x_t) - \eta \sum_{i=1}^d |\nabla f(x_t)|_i \gamma_{t,i} + 2\eta \sum_{i=1}^d \gamma_{t,i} |\nabla f (x_t)|_i\mathbb{I}[\sign(\nabla_i f(x_t)) \ne \sign(g_t)_i] + \frac{\eta^2}{2} \sum_{i=1}^d L_i \gamma_{t,i}^2,
\end{align*}
Conditioned on $\gF_{t-1}$, we take the expectation over $\xi_t$ on both sides (we write $\E_{\xi_t}[\cdot|\gF_{t-1}]$ as $\E[\cdot]$ for short), and use the property of $\Pr[\sign(\nabla_i f(x_t)) \ne \sign(g_t)_i] \le \frac{\sigma_{v,t,i}}{|\nabla f(x_t)|_i}$, this yield
\begin{align*}
    f(x_{t+1}) &\le \E[f(x_t)] - \eta \sum_{i=1}^d |\nabla f(x_t)|_i \gamma_{t,i} + 2\eta \sum_{i=1}^d\gamma_{t,i} \sigma_{v,t,i}   +  \frac{\eta^2}{2} \sum_{i=1}^d L_{v,i} \gamma_{t,i}^2 \\ 
    &\le \E[f(x_t)] - \eta \sum_{i=1}^d |\nabla f(x_t)|_i \gamma_{t,i} + \frac{2 \eta}{\sqrt{T}} \sum_{i=1}^d\gamma_{t,i} \sigma_{v,i}   +  \frac{\eta^2}{2} \sum_{i=1}^d L_{v,i} \gamma_{t,i}^2 \\
    &\le \E[f(x_t)] - \eta \norm{\nabla f(x_t)}_{1,\gamma_t} + \frac{2 \eta}{\sqrt{T}} \norm{\sigma_{v}}_{1,\gamma_t}  +  \frac{\eta^2}{2} \norm{L_{v}}_{1, \gamma_t^2},
\end{align*}
where the second inequality comes from $\sigma_{v,t,i}^2 \le \sigma_{v,i} / T $ when $n_t=T$ and the final inequality comes from the definition of $\gamma$-weighted norms. We then take the expectation over all previous $\xi_t, t\le T$, and take the telescope sum up to $T$ gives us
\begin{align*}
\sum_{t=1}^{T}\eta \norm{f(x_t)}_{1, \gamma_t} &\le f(x_1) - \E[f(x_{T+1})] + \frac{2\eta}{\sqrt{T}}\sum_{t=1}^{T} \norm{\sigma_{v}}_{1, \gamma_t} + \frac{\eta^2}{2}\sum_{t=1}^{T}\norm{L_{v}}_{1, \gamma^2_t}, \\
\frac{1}{T}\sum_{t=1}^{T} \norm{f(x_t)}_{1, \gamma_t} &\le \frac{f(x_1) - \E[f(x_{T+1})]}{\eta T} +\frac{2}{\sqrt{T}}\norm{\sigma_{v}}_{1, \gammahat} + \frac{\eta}{2}\norm{L_{v}}_{1, \gammahatsq},
\end{align*}
If we let $\Delta = f(x_1) - \min f(x)$ and $\eta =\frac{1}{\sqrt{T}}$, we have 
\begin{align*}
    \frac{1}{T}\sum_{t=1}^{T} \norm{f(x_t)}_{1, \gamma_t} \le \frac{1}{\sqrt{T}} \Big (\Delta + 2\norm{\sigma_{v}}_{1, \gammahat} + \frac{1}{2} \norm{L_{v}}_{1, \gammahatsq}\Big).
\end{align*}

\end{proof}

\subsection{Proof for \cref{thm:devasf}}
\label{appsec:proof_devasf}
\matconv*

\begin{proof}
The $L$-smooth assumption in \cref{ass:matrix_lip} yields the following descent inequality:
    \begin{align*}
        f(X_{t+1}) &\le f(X_t) + \inner{\nabla f(X_t)}{X_{t+1}-X_t} + \frac{1}{2} \norm{\sqrt{L_m}\odot (X_{t+1}-X_t)}_{F}^2 \\
        &= f(X_t) + \inner{ \nabla f(X_t)}{-\eta \Gamma_{t} \odot \msign (G_t)} + \frac{\eta^2}{2} \norm{\sqrt{L_m}\odot \Gamma_t \odot \msign(G_t)}_{F}^2,
    \end{align*}
    using $\norm{A \odot B}_F \le \norm{A}_F \norm{B}_F$, $\inner{A\odot B}{C}= \inner{B}{A\odot C}$, and $|\msign(G_t)_{ij}| = |\sum_{k=1}^{r} U_{ik} V_{jk}| \le \norm{U_{i,:}}_2 \cdot \norm{V_{j,:}}_2 \le 1$, we have,
    \begin{align*}
        f(X_{t+1}) &\le f(X_t) +  \inner{\Gamma_{t} \odot \nabla f(X_t)}{-\eta \msign (G_t)} + \frac{\eta^2}{2} \sum_{ij} L_{m,ij} \Gamma_{t,ij}
    \end{align*}
    observe that, in $\ell_{1}$ norm, we have $\norm{L_m\odot \Gamma_t^2 }_{\ell_1}=\norm{L_m}_{\ell_1,\Gamma_t^2}$, this yields
    \begin{align*}
        f(X_{t+1}) &\le f(X_t) +  \inner{\Gamma_{t} \odot \nabla f(X_t)}{-\eta \msign (G_t)} + \frac{\eta^2}{2} \norm{L_m}_{\ell, \Gamma_t^2}.
    \end{align*}
    now we let $\nabla f(X_t)= G_t - E_t$. $E_t$ is the error matrix and $\E_{\xi_t}[E_t | \gF_{t-1}]=0$ by the unbiased gradient assumption, plug this into the above equation, we have
    \begin{align*}
        f(X_{t+1}) &\le f(X_t) +  \inner{\Gamma_{t} \odot (G_t - E_t)}{-\eta \msign (G_t)} + \frac{\eta^2}{2} \norm{L_m}_{\ell, \Gamma_t^2}\\
        &\le f(X_t) + \inner{\Gamma_{t} \odot G_t}{-\eta \msign (G_t)} + \inner{\Gamma_{t} \odot E_t }{\eta \msign (G_t)} + \frac{\eta^2}{2} \norm{L_m}_{\ell, \Gamma_t^2}  \\
        &\le f(X_t) - \eta \inner{G_t}{\msign(G_t)}_{\Gamma_t} + \inner{\Gamma_{t} \odot E_t }{\eta \msign (G_t)} + \frac{\eta^2}{2} \norm{L_m}_{\ell, \Gamma_t^2}\\
        &\le f(X_t) - \eta\inner{G_t}{\msign(G_t)}_{\Gamma_t} + \eta \norm{\Gamma_t \odot E_t}_1 +\frac{\eta^2}{2} \norm{L_m}_{\ell, \Gamma_t^2},
    \end{align*}
    where the final inequality comes from the Holder inequality. Condition on $\gF_{t-1}$, now we take the expectation over $\xi_t$ on both sides (we write $\E_{\xi_t}[\cdot|\gF_{t-1}]$ as $\E[\cdot]$ for short), this yields
    \begin{align*}
        f(X_{t+1}) &\le \E[f(X_t)] - \eta \E[ \inner{G_t}{\msign(G_t)}_{\Gamma_t}  ] + \eta \E[\norm{\Gamma_t \odot E_t}_1] + \frac{ \eta^2}{2}\E[ \norm{L_m}_{\ell_{1},\Gamma_t^2}] \\
        &=\E[f(X_t)] - \eta \E[ \inner{G_t}{\msign(G_t)}_{\Gamma_t}  ] + \eta \E[\norm{\Gamma_t \odot E_t}_1] + \frac{ \eta^2}{2} \norm{L_m}_{\ell_{1},\Gamma_t^2}
    \end{align*}
     where the final equality uses $\{\Gamma_t\}$ is $\gF_{t-1}$ measurable. Now we want to study the error term  $\E[\norm{\Gamma_t \odot E_t}_1]$, we know the nuclear norm is upper bounded by the Frobenius norm times the square root of the rank of the matrix, now we have
    \begin{align*}
        \E[\norm{\Gamma_t \odot E_t}_1] \le \sqrt{r} \E[\norm{\Gamma_t \odot E_t}_F]
        \le \sqrt{r}\sqrt{\E[\norm{\Gamma_t \odot E_t}_F^2]} = \sqrt{r}\sqrt{\E[\sum_{i=1}^n\sum_{j=1}^m\Gamma_{t,ij}^2E_{t,ij}^2]},
    \end{align*}
    where is second inequality comes from the concavity of $\sqrt{X}$. Under the linearity of expectation and $\gF_{t-1}$-measurability of $\Gamma_t$, we have $\E[\sum_{i=1}^n\sum_{j=1}^m\Gamma_{t,ij}^2E_{t,ij}^2]=\sum_{i=1}^n\sum_{j=1}^m\Gamma_{t,ij}^2\sigma_{t,ij}^2$ (Here we simplify $\sigma_m$ to $\sigma$ for brevity). When $n_t=T$, we further have $\sigma_{t,ij}^2\le \sigma_{ij}^2/T$ and this gives us $\E[\norm{\Gamma_t \odot E_t}_1] \le \sqrt{\frac{r}{T}}\sqrt{\sum_{i=1}^n\sum_{j=1}^m \Gamma_{t,ij}^2\sigma_{ij}^2} = \sqrt{\frac{r}{T}} \norm{\sigma_m}_{F,\Gamma_t} $. Plug this back into the previous descent inequality, and it yields
    \begin{align*}
       f(X_{t+1}) \le \E[f(X_t)] - \eta \E[ \inner{G_t}{\msign(G_t)}_{\Gamma_t}  ] + \eta \sqrt{\frac{r}{T}} \norm{\sigma_m}_{F,\Gamma_t} + \frac{ \eta^2}{2} \norm{L_m}_{\ell_{1},\Gamma_t^2}.
    \end{align*}
    Take the expectation over previous $\xi_t, t\le T$ and apply the tower rule, and take the telescope sum of the above descent inequality up to $T$, which gives us
   \begin{align*}
\sum_{t=1}^{T}\eta \psi(G_t) &\le f(X_1) - \E[f(X_{T+1})] + \eta\sqrt{\frac{r}{T}}\sum_{t=1}^{T} \norm{\sigma_m}_{F,\Gamma_t}  + \frac{ \eta^2}{2}\sum_{t=1}^{T}\norm{L_m}_{\ell_{1},\Gamma_t^2}. \\
\frac{1}{T}\sum_{t=1}^{T} \psi(G_t) &\le \frac{f(X_1) - \E[f(X_{T+1})]}{\eta T} +\sqrt{\frac{r}{T}} \norm{\sigma_m}_{F,\Gammahat} +  \frac{\eta}{2}\norm{L_m}_{\ell_{1}, \Gammahatsq},
\end{align*}
If we let $\Delta = f(X_1) - \min f(X)$ and $\eta =\frac{1}{\sqrt{T}}$, we have 
\begin{align*}
    \frac{1}{T}\sum_{t=1}^{T} \psi(G_t) \le \frac{1}{\sqrt{T}} \Big (\Delta + \sqrt{r}\norm{\sigma_m}_{F,\Gammahat} + \frac{1}{2} \norm{L_m}_{\ell_{1},\Gammahatsq}\Big).
\end{align*}
\end{proof}

\newpage
\section{Experiment Details on Modded NanoGPT}
\label{appsec:modded_nanogpt}
\paragraph{Experiment Setup.} In our NanoGPT experiments, we utilized the May 27 snapshot of modded-nanogpt \citep{modded_nanogpt_2024}, which incorporates several modern architectural enhancements: rotary embeddings, RMSNorm, a linear decay scheduler, and $\text{ReLU}^2$ activations. This version also features an untied output head and three additional value embedding layers appended to the early transformer layer outputs. Each embedding block consists of approximately 38M parameters ($768 \times 50,257$); these four blocks total 152M parameters, resulting in a model size of 276M parameters. We train our model on the FineWeb-10B dataset \citep{penedo2024fineweb}. To stabilize training, we incorporate a sign-trick similar to that proposed in \citep{si2025adamuon} within the $\text{msign}$ function. We employ a warmup phase spanning 60\% of the total steps (at a fixed learning rate) for the linear decay schedule. The optimizer hyperparameters in \cref{fig:nano_val} and \cref{fig:nano_gpt_loss} are detailed in \cref{tab:hyper_nanogpt}.
\begin{table}[ht]
\centering
\caption{Hyperparameter configurations for the various optimizers evaluated in \cref{fig:nano_val} and \cref{fig:nano_gpt_loss}. For methods employing Nesterov momentum, the acceleration term is calculated as $\sign(\beta_1 m_{t} + (1-\beta_1) g_t)$ or $\msign(\beta_1 m_{t} + (1-\beta_1) g_t)$, where $m_t$ is the moving average and $g_t$ is the current gradient.}
\label{tab:hyper_nanogpt}
\begin{tabularx}{\textwidth}{l XXXXXX} 
\toprule
\textbf{Hyperparameter} & \textbf{Adam} & \textbf{Lion} & \textbf{MUON} & \textbf{SOAP} & \textbf{Adamuon} & \textbf{DeVA$_{S_{\infty}}$}   \\ 
\midrule
Learning Rate ($\eta$) & 0.001 & 0.001 & 0.001 & 0.001 & 0.001 & 0.001 \\
Weight Decay           & 0.1 & 0.1 & 0.1 & 0.1 & 0.1 & 0.1 \\
Batch Size (tokens)             & 49152   & 49152 & 49152 & 49152   & 49152 & 49152    \\
$\beta_1$ (First Moment)  & 0.9 & 0.95 & 0.95 & 0.95 & 0.95   & 0.95     \\
$\beta_2$ (Second Moment) & 0.99 & 0.95 &0.95 &0.95 &  0.95 & 0.95  \\
$\beta_3$ (Shampoo Moment) & -- & -- & 0.95 &  0.95 & -- & 0.95  \\
$\epsilon$ (Stability)    & $10^{-8}$ & $10^{-8}$ & $10^{-8}$ & $10^{-8}$ & $10^{-8}$ & $10^{-8}$ \\
\texttt{eigdecomp} Frequency   & -- & -- & -- & 10 & -- & 10 \\
Nesterov Acceleration   & False & True & False & False & False & False \\
\bottomrule
\end{tabularx}
\end{table}

\paragraph{Hyperparameter Robustness.} We further evaluate the hyperparameter robustness of $\devasf$ with respect to the learning rate and $\beta_3$ in \cref{fig:lr-sweep}, \cref{tab:lr-sweep}, and \cref{tab:beta3}. As shown in the learning-rate sweep results in \cref{fig:lr-sweep} and \cref{tab:lr-sweep}, $\devasf$ outperforms both SOAP and Muon when the learning rate is set to $10^{-3}$. Importantly, we also find that $\devasf$ is robust to the choice of $\beta_3$: all tested values of $\beta_3$ achieve validation losses lower than the best result obtained by SOAP (3.2769).

\begin{figure}[htbp]
  \centering
  \includegraphics[width=0.72\textwidth]{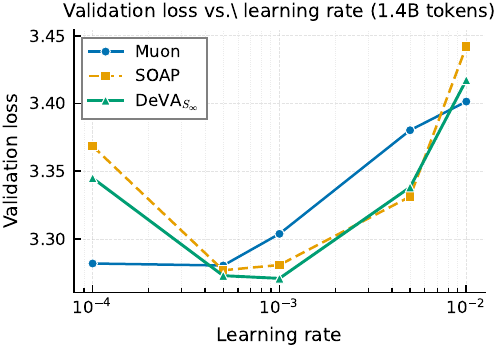}
  \caption{\textbf{Learning-rate sensitivity (1.4B tokens).} Validation loss versus learning rate for Muon, SOAP, and DeVA$_{S_{\infty}}$ under a fixed token budget.}
  \label{fig:lr-sweep}
\end{figure}

\begin{table}[htbp]
  \centering
  \caption{Numerical learning-rate sweep (1.4B tokens).}
  \label{tab:lr-sweep}
  \begin{tabular}{lccccc}
    \toprule
    Optimizer & LR $=10^{-4}$ & LR $=5\times10^{-4}$ & LR $=10^{-3}$ & LR $=5\times10^{-3}$ & LR $=10^{-2}$ \\
    \midrule
    Muon & 3.2819 & \textbf{3.2804} & 3.3038 & 3.3803 & 3.4015 \\
    SOAP & 3.3686 & \textbf{3.2769} & 3.2808 & 3.3313 & 3.4422 \\
    DeVA$_{S_{\infty}}$ & 3.3450 & 3.2729 & \textbf{3.2709} & 3.3382 & 3.4174 \\
    \bottomrule
  \end{tabular}
\end{table}

\begin{table}[htbp]
  \centering
  \caption{\textbf{Robustness of DeVA$_{S_{\infty}}$ to $\beta_3$.} Validation loss when varying $\beta_3$ with the learning rate fixed to $10^{-3}$ (1.4B-token setting; same evaluation as above). The bold entry is the lowest loss in this $\beta_3$ grid and corresponds to the configuration we emphasize in the rebuttal.}
  \label{tab:beta3}
  \begin{tabular}{lccccc}
    \toprule
    $\beta_3$ & 0.8 & 0.9 & 0.95 & 0.975 & 0.9875 \\
    \midrule
    Validation loss & 3.2719 & \textbf{3.2687} & 3.2743 & 3.2749 & 3.2723 \\
    \bottomrule
  \end{tabular}
\end{table}

\newpage

\section{Experiment Details on ImageNet-1K}
\label{appsec:imagenet-1k}
\paragraph{Experiment Setup.} For our ImageNet experiments, we train a transformer model ViT-L/16 on the ImageNet-1K dataset from ILSVRC2012 \citep{deng2009imagenet}. We compare convergence behavior using early stopping within a five-epoch window and no learning rate decay. To stabilize training, we incorporate a sign-trick similar to that proposed in \citep{si2025adamuon} within the $\text{msign}$ function. We observe that a sign-stabilized $\msign$ update is beneficial for image classification tasks. The hyperparameter configurations corresponding to \cref{fig:imagenet_loss} are detailed in \cref{tab:hyper_VIT}.
\begin{table}[ht]
\centering
\caption{Hyperparameter configurations for the various optimizers evaluated in \cref{fig:imagenet_loss}. For methods employing Nesterov momentum, the acceleration term is calculated as $\sign(\beta_1 m_{t} + (1-\beta_1) g_t)$ or $\msign(\beta_1 m_{t} + (1-\beta_1) g_t)$, where $m_t$ is the moving average and $g_t$ is the current gradient.}
\label{tab:hyper_VIT}
\begin{tabularx}{\textwidth}{l XXXXX} 
\toprule
\textbf{Hyperparameter} & \textbf{Adam} & \textbf{MUON} & \textbf{SOAP} & \textbf{Adamuon} & \textbf{DeVA$_{S_{\infty}}$} \\ 
\midrule
Learning Rate ($\eta$) & 0.001 & 0.001 & 0.001 & 0.001 & 0.001 \\
Weight Decay           & 0.1 & 0.1 & 0.1 & 0.1 & 0.1 \\
Batch Size (imgs)             & 512   & 512    & 512 & 512 & 512    \\
$\beta_1$ (First Moment)  & 0.9 & 0.95   & 0.95 & 0.95 & 0.95   \\
$\beta_2$ (Second Moment) & 0.999 & 0.95 &  0.95 & 0.95 & 0.95  \\
$\beta_3$ (Shampoo Moment) & -- & -- &  0.95 & -- & 0.95  \\
$\epsilon$ (Stability)    & $10^{-8}$ & $10^{-8}$ & $10^{-8}$ & $10^{-8}$ & $10^{-8}$ \\
\texttt{eigdecomp} Frequency   & -- & -- & 10 & -- & 10 \\
Nesterov Acceleration   & False & True & False & False & True \\
\bottomrule
\end{tabularx}
\end{table}

\newpage

\section{Experiment Details on CIFAR-10}
\label{appsec:cifar-10}
\paragraph{Experiment Setup.} For our CIFAR-10 experiments, we compare the validation accuracy of ResNet-20 over 40 epochs using a linear learning rate decay without any warm-ups. The hyperparameter configurations for each optimizer—selected based on peak validation accuracy in \cref{fig:cifar10_optimal_lrs}—are provided in \cref{tab:hyper_resnet}.

\begin{table}[ht]
\centering
\caption{Best Hyperparameter settings across different optimizers on ResNet-20 in \cref{fig:cifar10_optimal_lrs}. For methods employing Nesterov momentum, the acceleration term is calculated as $\sign(\beta_1 m_{t} + (1-\beta_1) g_t)$ or $\msign(\beta_1 m_{t} + (1-\beta_1) g_t)$, where $m_t$ is the moving average and $g_t$ is the current gradient.}
\label{tab:hyper_resnet}
\begin{tabularx}{\textwidth}{l XXXXXXX} 
\toprule
\textbf{Hyperparameter} & \textbf{Adam} & \textbf{Signum} & \textbf{DeVA$_{\ell_{\infty}}$}  &  \textbf{MUON} & \textbf{SOAP} &  \textbf{Adamuon} & \textbf{DeVA$_{S_{\infty}}$} \\ 
\midrule
Learning rate ($\eta$) & 0.05 & 0.05 & 0.05 & 0.001 & 0.001 & 0.001 & 0.001 \\ 
Weight Decay           & 0.1 & 0.1 & 0.1 & 0.1 & 0.1 & 0.1 & 0.1 \\
Batch Size (imgs)             & 128   & 128    & 128 & 128 & 128 & 128 & 128   \\
$\beta_1$ (First Moment)  & 0.9 & 0.9 & 0.9 & 0.95   & 0.95 & 0.95 & 0.95    \\
$\beta_2$ (Second Moment) & 0.999 & 0.999 & 0.999 & 0.95 & 0.95 & 0.95 & 0.95  \\
$\beta_3$ (Shampoo Moment) & -- & -- & -- & -- & 0.95 &  -- & 0.95  \\
$\epsilon$ (Stability)    & $10^{-8}$ & $10^{-8}$ & $10^{-8}$ & $10^{-8}$ & $10^{-8}$ & $10^{-8}$ & $10^{-8}$  \\
\texttt{eigdecomp} Frequency   & -- & -- & -- & -- & 5 & -- &  5 \\
Nesterov Acceleration   & False & False & False & False & False & False & False \\
\bottomrule
\end{tabularx}
\end{table}

\newpage
\section{Additional Results}
We summarize the per-step computational and memory costs of the matrix optimizers Muon, SOAP, and $\devasf$ in \cref{tab:complexity} of \cref{appsec:mem_run}. The main runtime bottleneck of $\devasf$ is the additional eigendecomposition step, while the memory bottleneck comes from storing the covariance matrix and second-order moment matrices. In \cref{appsec:abl}, we discuss how to reduce the runtime complexity by increasing the batch size and decreasing the eigendecomposition frequency. In \cref{appsec:second_ord}, we investigate whether a reduced second-order moment estimator can achieve performance comparable to the original $\devasf$. Finally, in \cref{appsec:snr}, we present the evolution of the signal-to-noise ratio (SNR) across different layers in \cref{tab:snr}, observing that the SNR peaks during the early stage of training and remains stable as training progresses.

\subsection{Memory and runtime overhead.}
\label{appsec:mem_run}
\begin{table}[htbp]
  \centering
  \caption{\textbf{Asymptotic per-step cost (matrix shape $n\times m$).} Memory footprint and floating-point work per optimizer step, using the same symbols and conventions as the main text. The symbol \emph{freq} denotes how often expensive subroutines (for example, eigendecompositions) are invoked relative to the baseline step; these terms appear explicitly so that amortization over an inner loop is not hidden. The last column identifies the dominant subroutine for each method. This table is included to make the complexity discussion in our response self-contained.}
  \label{tab:complexity}
  \resizebox{\textwidth}{!}{
  \begin{tabular}{llll}
    \toprule
    Optimizer & Memory & FLOPs per step & Dominant subroutine \\
    \midrule
    Muon & $\mathcal{O}(5nm^2)$ & $\mathcal{O}(3nm)$ & Newton--Schulz iteration \\
    SOAP & $\mathcal{O}\left(n^2m + \frac{n^3 + m^3}{\text{freq}}\right)$ & $\mathcal{O}(n^2 + m^2 + 4nm)$ & Eigendecomposition; covariance \\
    DeVA$_{S_{\infty}}$ & $\mathcal{O}\left(n^2m + 5nm^2 + \frac{n^3 + m^3}{\text{freq}}\right)$ & $\mathcal{O}(n^2 + m^2 + 4nm)$ & Eigendecomposition; covariance; N--S \\
    \bottomrule
  \end{tabular}}
\end{table}

\subsection{Ablations}
\label{appsec:abl}

\begin{figure}[h]
    \centering
    \includegraphics[width=0.95\linewidth]{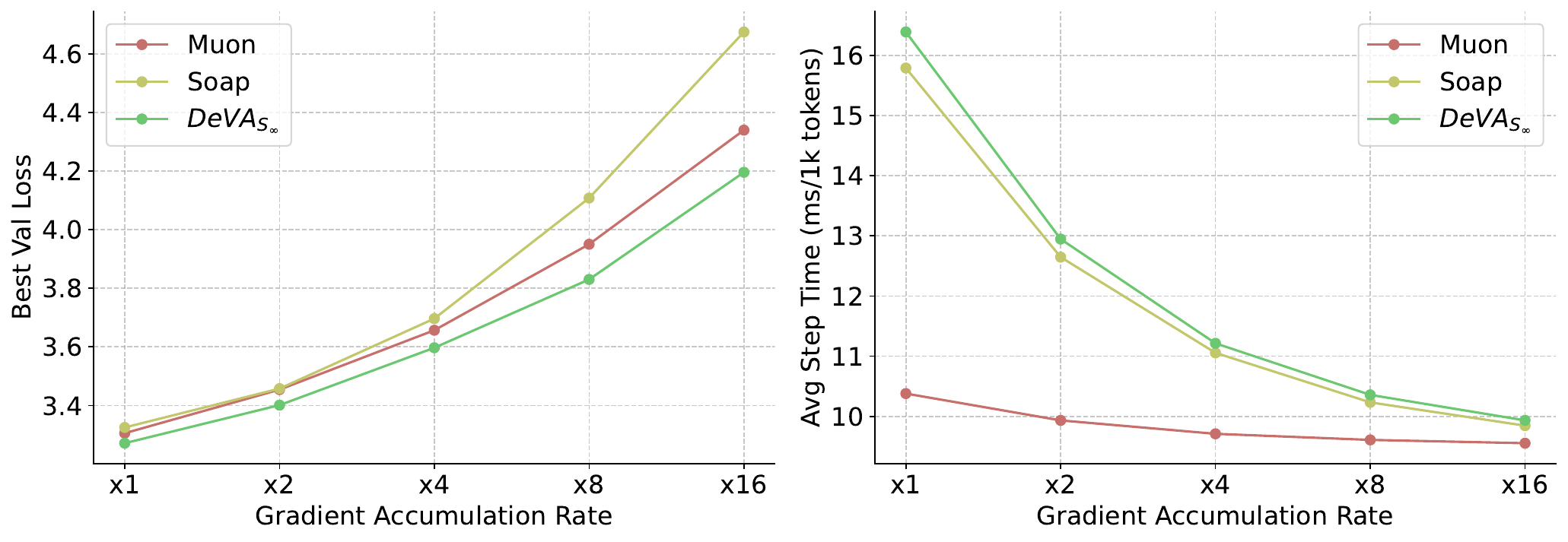}
    \caption{Batch size sensitivity on NanoGPT (274M). The generalization performance of $\devasf$ remains robust across larger batch sizes. Furthermore, the performance gap narrows the difference in average runtime between Muon and $\devasf$ at higher batch scales.}
    \label{fig:abl_gcc}
\end{figure}

\paragraph{Different batch sizes} We evaluate the impact of batch size by varying the number of gradient accumulation steps. As shown in \cref{fig:abl_gcc}, validation loss increases across all optimizers as the gradient accumulation rate rises; notably, SOAP exhibits the most significant performance degradation. Most importantly, $\devasf$ consistently achieves the lowest validation loss across all accumulation settings. Muon performs competitively, generally maintaining a performance profile between $\devasf$ and SOAP.

Regarding computational efficiency, Muon remains the fastest optimizer in terms of average step time across all accumulation levels. For all methods, average step time decreases as the accumulation rate increases, likely due to the amortization of optimizer overhead (such as \texttt{eigdecomp} and $\text{msign}$) over a larger number of micro-batches. While $\devasf$ and SOAP incur higher overhead than Muon at low accumulation rates (e.g., $\text{rate}=1$), this gap narrows significantly as the accumulation rate increases.

\begin{figure}[h]
    \centering
    \includegraphics[width=0.95\linewidth]{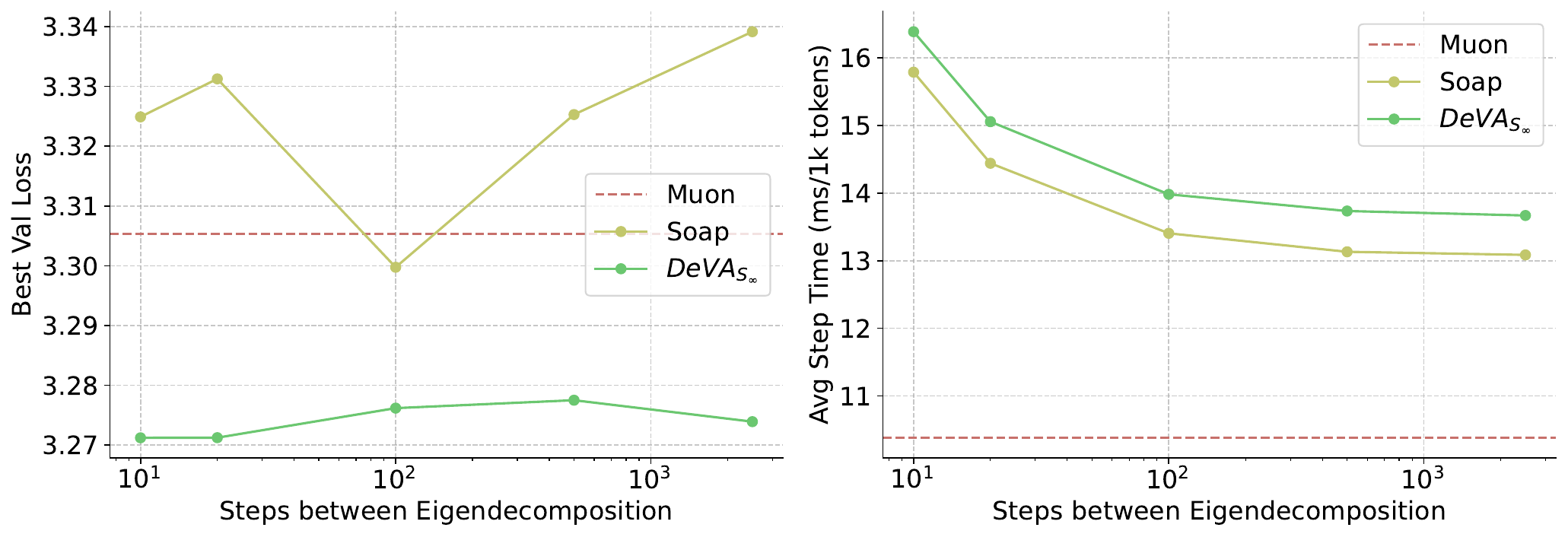}
    \caption{Sensitivity to eigen-decomposition frequency on NanoGPT (274M). $\devasf$ maintains consistently low validation loss even as the frequency of eigen-decomposition is reduced, demonstrating its robustness to stale updates. }
    \label{fig:abl_freq}
\end{figure}

\paragraph{Robustness to eigen-decomposition frequency.} As illustrated in \cref{fig:abl_freq}, $\devasf$ exhibits remarkable stability; its validation loss remains consistently lower than that of SOAP, fluctuating only minimally (around $3.27$) even as the eigen-decomposition frequency is significantly reduced. While SOAP maintains a slightly higher validation loss (between $3.30$ and $3.34$), it also demonstrates relative resilience to less frequent updates.

There is a pronounced efficiency gain as the eigen-decomposition interval increases. Both optimizers experience a sharp reduction in average step time when moving from a frequency of 10 to 100 steps, with diminishing returns for higher intervals. This confirms that reducing the frequency of the eigen-decomposition step significantly lowers computational overhead without substantially compromising model convergence.

\subsection{Second Moment Accumulation in $\devasf$}
\label{appsec:second_ord}
\begin{figure}[h]
    \centering
    \includegraphics[width=0.95\linewidth]{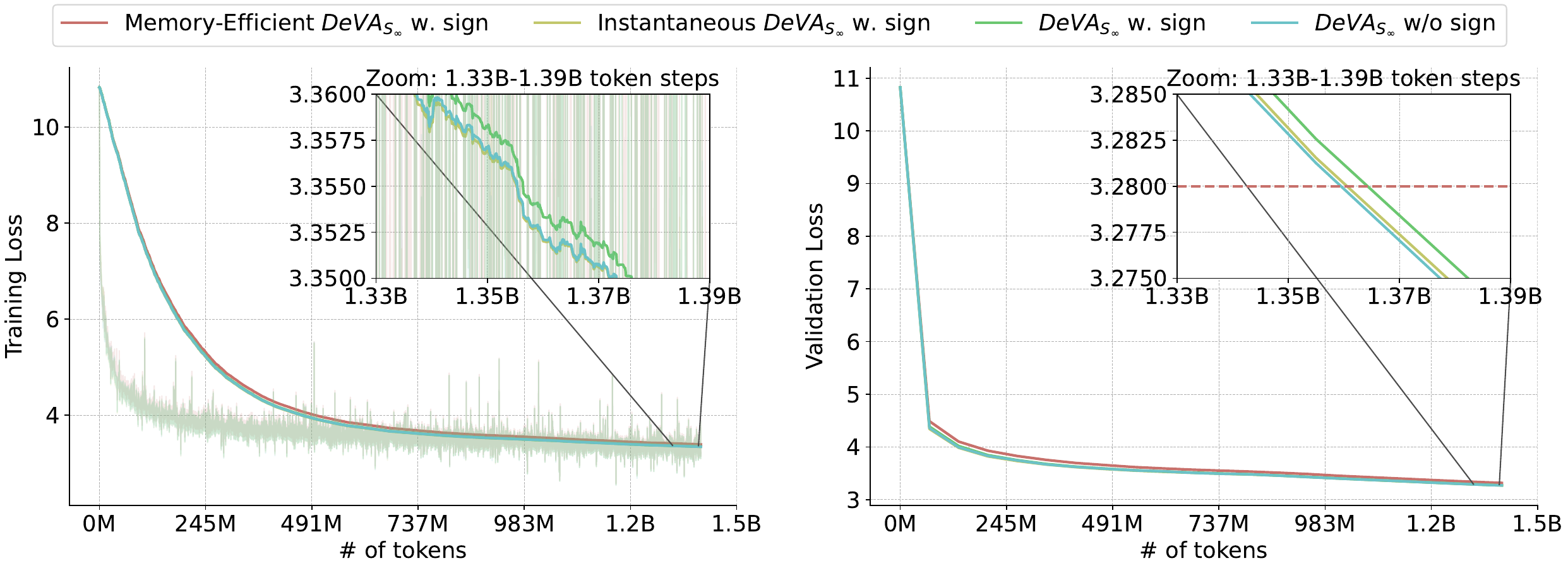}
    \caption{Training and Validation Performance of different $\devasf$ variants on NanoGPT (274M).}
    \label{fig:sec_mom}
\end{figure}

In this study, we evaluate two variants of $\devasf$ distinguished by their second-moment accumulation strategies:

\begin{itemize}

\item \textbf{Variant (a) (Memory-Efficient $\devasf$, \cref{alg:deva_sinfty_eff}):} Following the approach in Adafactor \citep{shazeer2018adafactor}, which achieves sublinear memory complexity by separately accumulating the squared row and column sums of the gradient $G_t$, we develop a memory-efficient second-order moment for $\devasf$. By independently accumulating row sums $r_t$ and column sums $c_t$ (Lines 14--15, \cref{alg:deva_sinfty_eff}), we reduce the memory overhead of the second-order moments from $O(nm)$ to $O(n+m)$.

\item \textbf{Variant (b) (Instantaneous $\devasf$, \cref{alg:deva_sinfty_ins}):} 
In contrast to calculating the second-order moments $r_t$ and $c_t$ by EMA of $G_t'$, we introduce an Adam-style variant that calculates them based on the sample gradient $G_t'$ (Lines 12--13).
\end{itemize}
Beyond these accumulation strategies, we also explore the impact of the sign-stabilization trick within the $\text{msign}$ function. While previously shown to benefit ImageNet classification, we evaluate its general utility on language modeling. We append the suffixes \textit{w. sign} and \textit{w/o sign} to $\devasf$ to denote whether the $\text{msign}$ operator includes this stabilization.

\begin{table}[h]
    \centering
    \caption{Best Validation loss across $\devasf$ variants on NanoGPT (274M).}

    \begin{tabularx}{\linewidth}{l XXXX}
    \toprule
         & \mbox{$\devasf$ \textit{w. sign}} & \mbox{$\devasf$\textit{ w/o sign}} &  \mbox{Eff. $\devasf$ \textit{w. sign}} & \mbox{Inst. $\devasf$ \textit{w. sign}} \\ \midrule
     Best Val Loss    & 3.271 & 3.269 & 3.309 & 3.270\\   
     \bottomrule
    \end{tabularx}

    \label{tab:sec_mom}
\end{table}

The training and validation curves are illustrated in \cref{fig:sec_mom}, with detailed validation results summarized in \cref{tab:sec_mom}. Our results demonstrate that employing sublinear memory for second-order momentum incurs only a minimal performance degradation (from 3.271 to 3.309). Furthermore, Instantaneous $\devasf$ yields a marginal improvement over the baseline $\devasf$. Notably, in contrast to common findings in image classification tasks, sign-stabilized $\devasf$ adversely affects the final validation performance.

\begin{algorithm}[!h]
\caption{Memory-Efficient $\devasf$}\label{alg:deva_sinfty_eff}
\begin{algorithmic}[1]
\INPUT $\beta_1, \beta_2, \beta_3, freq$, $\epsilon, T$, $\eta_t$ for $t=1,...,T$
\OUTPUT $X_T \in \R^{n \times m}$
\STATE Initialize $L_0 \in \R^{n \times n}, R_0 \in \R^{m \times m}, M_0 \in \R^{n \times m}, V_{r,0} \in \R^{n}, V_{c,0} \in \R^{m} $
\STATE Randomly initialize $X_1 \in \R^{n \times m}$
\FOR{$t=1,...,T$}
\STATE $G_{t} = \nabla f(X_{t},\xi_{t}) $
\STATE $L_t = \beta_3 L_{t-1} + (1-\beta_3) G_tG_t^{T}$
\STATE $R_t=\beta_3 R_{t-1} + (1-\beta_3)G_t^{T}G_t$
\IF{$t \mod freq == 0$} 
\STATE $Q_L,Q_R = \texttt{eigdecomp}(L_t), \texttt{eigdecomp}(R_t)$ 
\ENDIF
\STATE $G_t'=Q_L^{T}G_tQ_R$
\STATE $M_{t}= \beta_1 M_{t-1} + (1-\beta_1)G_{t}' $
\STATE $r_t = \left( \sqrt{\sum_{j=1}^m (M_t)_{ij}^2} \right)_{i=1}^n$ 
\STATE $c_t = \left( \sqrt{\sum_{i=1}^n (M_t)_{ij}^2} \right)_{j=1}^m$
{\color{blue}\STATE $V_{r,t} = \beta_2 V_{r,t-1} + (1-\beta_2) r_t$ 
\STATE $V_{c,t} = \beta_2 V_{c,t-1} + (1-\beta_2) c_t$ 
\STATE $\Gamma_t = \left(V_{r,t}V_{c,t}^{T} \oslash (r_t c_t^{T} +\epsilon)\right)^{-1/2}$}
\STATE $D_t =\Gamma_t \odot \msign(M_t) \odot  0.2 \sqrt{\max (n,m)}$
\STATE $X_{t+1} = X_t - \eta_t D_t$
\ENDFOR 
\end{algorithmic}
\end{algorithm}

\begin{algorithm}[!h]
\caption{Instantaneous $\devasf$}\label{alg:deva_sinfty_ins}
\begin{algorithmic}[1]
\INPUT $\beta_1, \beta_2, \beta_3, freq$, $\epsilon, T$, $\eta_t$ for $t=1,...,T$
\OUTPUT $X_T \in \R^{n \times m}$
\STATE Initialize $L_0 \in \R^{n \times n}, R_0 \in \R^{m \times m}, M_0 \in \R^{n \times m}, V_{r,0} \in \R^{n}, V_{c,0} \in \R^{m} $
\STATE Randomly initialize $X_1 \in \R^{n \times m}$
\FOR{$t=1,...,T$}
\STATE $G_{t} = \nabla f(X_{t},\xi_{t}) $
\STATE $L_t = \beta_3 L_{t-1} + (1-\beta_3) G_tG_t^{T}$
\STATE $R_t=\beta_3 R_{t-1} + (1-\beta_3)G_t^{T}G_t$
\IF{$t \mod freq == 0$} 
\STATE $Q_L,Q_R = \texttt{eigdecomp}(L_t), \texttt{eigdecomp}(R_t)$ 
\ENDIF
\STATE $G_t'=Q_L^{T}G_tQ_R$
\STATE $M_{t}= \beta_1 M_{t-1} + (1-\beta_1)G_{t}' $
\STATE $r_t = \left( \sqrt{\sum_{j=1}^m ({\color{blue} G'_t})_{ij}^2} \right)_{i=1}^n$ 
\STATE $c_t = \left( \sqrt{\sum_{i=1}^n ({\color{blue}G'_t})_{ij}^2} \right)_{j=1}^m$
\STATE $V_{t} = \beta_2 V_{t-1} + (1-\beta_2) r_t c_t^{T}$ 
\STATE $\Gamma_t = \left(V_{t} \oslash (r_t c_t^{T} +\epsilon)\right)^{-1/2}$
\STATE $D_t =\Gamma_t \odot \msign(M_t) \odot  0.2 \sqrt{\max (n,m)}$
\STATE $X_{t+1} = X_t - \eta_t D_t$
\ENDFOR 
\end{algorithmic}
\end{algorithm}

\subsection{Signal-to-noise ratio during pretraining}
\label{appsec:snr}

\begin{table}[htbp]
  \centering
  \caption{\textbf{SNR matrix evolution (Frobenius norm).} Frobenius norm of the estimated signal-to-noise ratio (SNR) matrix at representative early- and late-depth layers and at the listed pretraining checkpoints. 
  \textbf{Boldface} in each row is the largest norm observed for that layer across the five token counts. The norms generally shift quickly within the first few million tokens, consistent with rapid reorganization of gradient statistics while representations are still forming. Thereafter, the norms vary within a comparatively narrow band, consistent with an approximately stationary signal-to-noise balance in late training.}
  \label{tab:snr}

  \begin{tabular}{lccccc}
    \toprule
    Layer & 5M tokens & 10M tokens & 50M tokens & 100M tokens & 500M tokens \\
    \midrule
    \texttt{block.0.mlp.c\_proj} & \textbf{1973.9} & 1634.4 & 1534.6 & 1635.8 & 1652.0 \\
    \texttt{block.0.attn.c\_proj} & \textbf{969.5} & 910.7 & 753.4 & 847.5 & 854.8 \\
    \texttt{block.11.mlp.c\_proj} & 499.1 & \textbf{2367.3} & 1651.1 & 1673.4 & 1747.4 \\
    \texttt{block.11.attn.c\_proj} & 373.4 & \textbf{1210.6} & 781.2 & 817.7 & 894.7 \\
    \bottomrule
  \end{tabular}

\end{table}


\end{document}

%% file: commands.tex
\def\eqref#1{equation~\ref{#1}}

\def\1{\bm{1}}

\DeclareMathAlphabet{\mathsfit}{\encodingdefault}{\sfdefault}{m}{sl}
\SetMathAlphabet{\mathsfit}{bold}{\encodingdefault}{\sfdefault}{bx}{n}

\def\gB{{\mathcal{B}}}

\def\gF{{\mathcal{F}}}

\def\gH{{\mathcal{H}}}

\newcommand{\E}{\mathbb{E}}

\newcommand{\R}{\mathbb{R}}

\newcommand{\Var}{\mathrm{Var}}

\DeclareMathOperator*{\argmin}{\text{arg\,min}}
\DeclareMathOperator*{\eig}{\text{eig}}

\renewcommand{\epsilon}{\varepsilon}

\DeclareMathOperator{\sign}{\text{sign}}
\DeclareMathOperator{\msign}{\text{msign}}
\DeclareMathOperator{\fvec}{\text{vec}}
\DeclareMathOperator{\diag}{\text{diag}}
\DeclareMathOperator{\deva}{\text{DeVa}}
\DeclareMathOperator{\devasf}{\text{DeVA}_{S_{\infty}}}
\DeclareMathOperator{\devalf}{\text{DeVA}_{\ell_{\infty}}}

\DeclareMathOperator{\ema}{\text{EMA}}
\DeclareMathOperator{\Tr}{\text{Tr}}
\DeclareMathOperator{\gammahat}{\hat{\gamma}}
\DeclareMathOperator{\gammahatsq}{\hat{\gamma}^2}
\DeclareMathOperator{\Gammahat}{\hat{\Gamma}}
\DeclareMathOperator{\Gammahatsq}{\hat{\Gamma}^2}

\global\long\def\norm#1{\|#1\|}%
\global\long\def\inner#1#2{\langle#1,#2\rangle}%